\documentclass{article}

\usepackage{microtype}
\usepackage{graphicx}
\usepackage{subfigure}
\usepackage{booktabs} 
\usepackage{multirow}
\usepackage{hyperref}

\usepackage[accepted]{icmlarxiv}

\usepackage{bm}
\usepackage{amsmath}
\usepackage{amssymb}
\usepackage{amsthm}
\newtheorem{theorem}{Theorem}[section]

\newtheorem{property}[theorem]{Property}
\newtheorem{assumption}[theorem]{Assumption}
\newtheorem{corollary}[theorem]{Corollary}
\newtheorem{definition}[theorem]{Definition}
\usepackage{wrapfig}

\icmltitlerunning{Hawkes Processes on Graphons}

\begin{document}

\newcommand{\xu}[1]{{\color{red} xu: #1}}
\newcommand{\zha}[1]{{\color{blue} Zha: #1}}

\twocolumn[
\icmltitle{Hawkes Processes on Graphons}



\icmlsetsymbol{equal}{*}

\begin{icmlauthorlist}
\icmlauthor{Hongteng Xu}{ruc,blab}
\icmlauthor{Dixin Luo}{bit}
\icmlauthor{Hongyuan Zha}{cuhk}
\end{icmlauthorlist}

\icmlaffiliation{ruc}{Gaoling School of Artificial Intelligence, Renmin University of China, Beijing, China}
\icmlaffiliation{blab}{Beijing Key Laboratory of Big Data Management and Analysis Methods, Beijing, China}
\icmlaffiliation{bit}{School of Computer Science and Technology, Beijing Institute of Technology, Beijing, China}
\icmlaffiliation{cuhk}{School of Data Science, Shenzhen Research Institute of Big Data, The Chinese University of Hong Kong, Shenzhen, China}

\icmlcorrespondingauthor{Dixin Luo}{dixin.luo@bit.edu.cn}

\icmlkeywords{Hawkes process, graphon, hierarchical optimal transport, heterogeneous event sequences}

\vskip 0.3in
]



\printAffiliationsAndNotice{}  

\begin{abstract}
We propose a novel framework for modeling multiple multivariate point processes, each with heterogeneous event types that share an underlying space and obey the same generative mechanism.
Focusing on Hawkes processes and their variants that are associated with Granger causality graphs, our model leverages an uncountable event type space and samples the graphs with different sizes from a nonparametric model called {\it graphon}. 
Given those graphs, we can generate the corresponding Hawkes processes and simulate event sequences. 
Learning this graphon-based Hawkes process model helps to 1) infer the underlying relations shared by different Hawkes processes; and 2) simulate event sequences with different event types but similar dynamics.
We learn the proposed model by minimizing the hierarchical optimal transport distance between the generated event sequences and the observed ones, leading to a novel reward-augmented maximum likelihood estimation method. 
We analyze the properties of our model in-depth and demonstrate its rationality and effectiveness in both theory and experiments. 
\end{abstract}

\section{Introduction}
As a powerful statistical tool, Hawkes process~\cite{hawkes1971spectra} has been widely used to model event sequences in the continuous-time domain. 
Suppose that we have an event sequence $\{(t_i, v_i)\in [0, T]\times \mathcal{V}\}_{i=1}^{N}$, where $[0, T]$ is the observation time window, $\mathcal{V}$ is the set of event types, and $(t_i, v_i)$ is the $i$-th event at time $t_i$ with type $v_i$. 
Equivalently, we can represent the sequence by a counting process $\bm{N}(t) = \{N_v(t)\}_{v\in\mathcal{V}}$, where $N_v(t)$ is the number of the type-$v$ events till time $t$.
A Hawkes process characterizes the expected instantaneous rate of occurrence of the type-$v$ event at time $t$ by a conditional intensity function~\cite{liniger2009multivariate}: for $v\in\mathcal{V}$ and $t\in[0, T]$, 
\begin{eqnarray}\label{eq:hp}
\begin{aligned}
\lambda_v(t) :=\frac{\mathbb{E}[\text{d}N_v(t)|\mathcal{H}_t]}{\text{d}t}= \mu_v + \sideset{}{_{t_i< t}}\sum \phi_{vv_i}(t,t_i).
\end{aligned}
\end{eqnarray}
Here, $\mathcal{H}_t=\{(t_i,v_i)|t_i<t\}$ contains the past events till time $t$. 
$\mu_v\geq 0$ is the base rate of type-$v$ event. 
$\{\phi_{vv'}(t,t')\geq 0\}_{v,v'\in\mathcal{V},t'<t}$ are the so called \textit{impact functions}, and $\phi_{vv'}(t,t')$ quantifies the influence of the type-$v'$ event at time $t'$ on the type-$v$ event at time $t$. 
Accordingly, $\sum_{t_i<t} \phi_{vv_i}(t,t_i)$ accumulates the impacts of the past events. 
The set of impact functions gives rise to the \textit{Granger causality graph} of the event types~\cite{eichler2017graphical,xu2016learning}, denoted as $G(\mathcal{V},\mathcal{E})$ --- an edge $v'\rightarrow v\in\mathcal{E}$ means that a past type-$v'$ event can trigger the occurrence of a type-$v$ event in the future, and $v'\rightarrow v\notin \mathcal{E}$ if and only if $\phi_{vv'}(t,t')\equiv 0$. 

Hawkes process, together with the corresponding Granger causality graph of event types, has become instrumental for many applications involving event sequences, such as social network modeling~\cite{farajtabar2017coevolve} and financial data analysis~\cite{bacry2015hawkes}. 
Interestingly, even with recent models enhancing Hawkes processes with deep neural networks~\cite{mei2017neural,zhang2020self,zuo2020transformer}, the work in~\cite{tank2018neural} shows that the group sparsity of their neural networks' parameters can still be interpreted by Granger causality of the event types.

\begin{figure}[!t]
    \centering
    \includegraphics[width=0.85\linewidth]{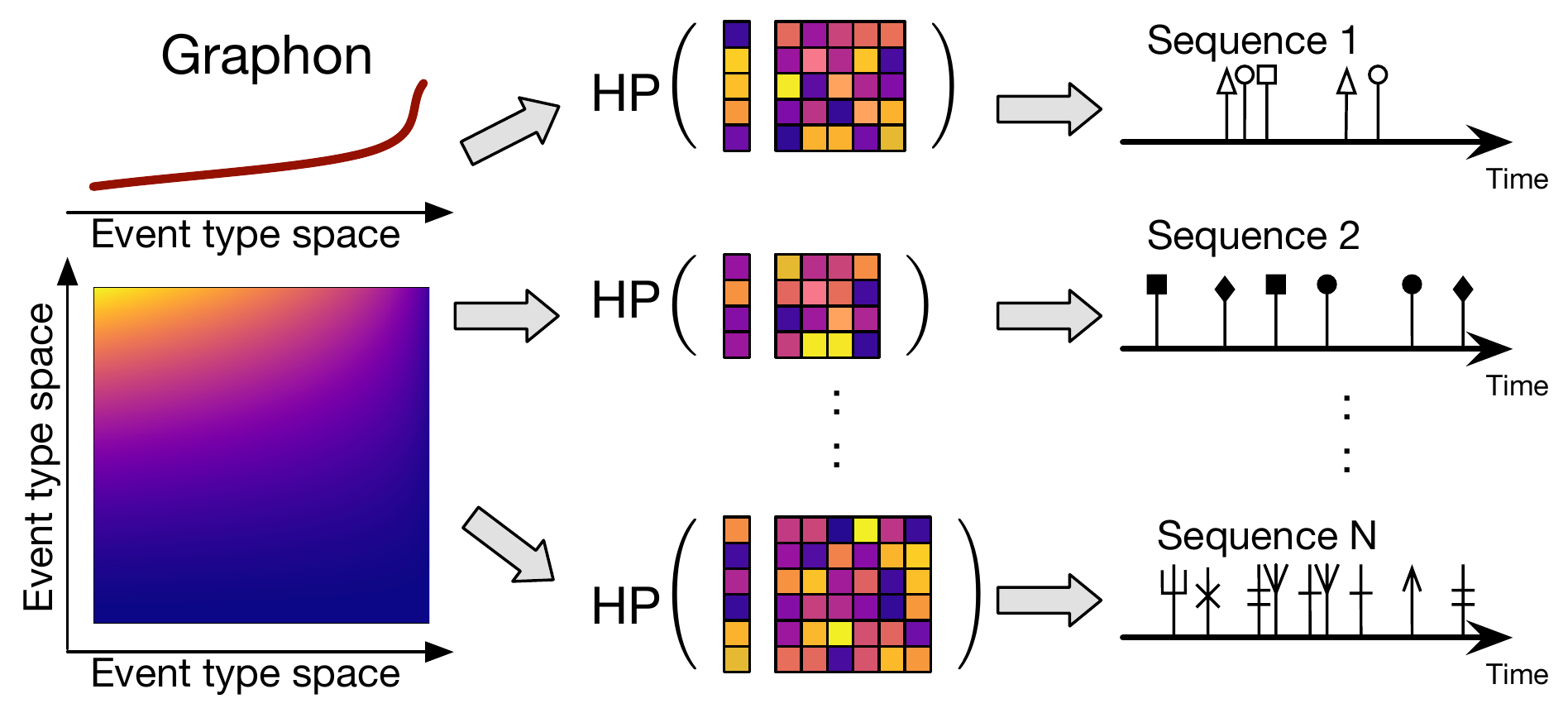}
    \vspace{-4mm}
    \caption{\small{An illustration of the Hawkes processes on a graphon.}}
    \label{fig:ghp}
\end{figure}

Despite achieving many successes, the applications of the Hawkes-related processes are limited for \textit{homogeneous} scenarios in which all the event sequences are generated by one point process defined on a known set of event types.
Although some methods consider learning multiple point processes for the sequences in different clusters~\cite{luo2015multi,xu2017dirichlet} or time periods~\cite{lin2016infinite,alaa2017learning}, they still maintain a single set of event types. 
This setting, however, is in conflict with the \textit{heterogeneous} nature of many real-world event sequences --- the event types are often sampled from an underlying event type space, and new sequences are driven by the latent sampling process and are generated with event types unobserved before.
Accordingly, for different event sequences, their point processes are defined with different event types, and thus, obey different generative mechanisms.
We illustrate this phenomenon via event sequences defined on networks. 

\textbf{Social networks}. 
Users of different networks, $e.g.$, Facebook and Twitter, are actually sampled from the same underlying populations ($i.e.$, all the Internet users in the world). 
When using Hawkes processes to model the user behaviors on those different networks~\cite{blundell2012modelling,zhou2013learning,zhao2015seismic}, the respective users are considered as event types and their corresponding Granger causality graphs can be treated as different subgraphs sampled from a large latent graph for the whole population. 
Additionally, with the entering of new users and the exiting of old ones, those networks are time-varying and their corresponding Hawkes processes at different time are different.

\textbf{Patient admissions}.
For a patient suffering from several diseases, his admissions in a hospital over time are often assumed to be driven by the Granger causality graph of his diseases ($i.e.$, disease graph), and thus, modeled by a Hawkes process~\cite{xu2017learning}.
For patients with different profiles, even for the same patient in different age periods, their disease graphs and the corresponding Hawkes processes can be very different. 
The diseases appearing in each Hawkes process are sampled from the same set of diseases, $e.g.$, the international classification of diseases (ICD), and each specific disease graph is a subgraph of an unknown graph constructed by all the diseases in the set. 
Moreover, with the development of biomedical science, we may find new diseases and observe new admissions in the future. 

Besides these two typical examples, the sequential shopping behaviors on different platforms, the transactions of stocks in different markets, and the diffusion of a virus in different cities, $etc.$, all these event sequences are heterogeneous, whose event types can better be modeled as samples from an underlying infinite even uncountable \textit{event type space}. 
When modeling such event sequences, we need to learn a generative model for their point processes beyond just learning a single point process for each of them individually.

To this end, we propose a new {\it graphon}-based Hawkes process (GHP). 
Essentially, our GHP is a hierarchical generative model for a collection of Hawkes processes with heterogeneous types (and their variants). 
As illustrated in Figure~\ref{fig:ghp}, it not only models the generative mechanisms of event sequences by Hawkes processes but also designs a \textit{graphon} model~\cite{lovasz2012large} to generate the event types of the different Hawkes processes from an uncountable event type space.
By sampling the graphon, we generate the parameters of various Hawkes processes and simulate event sequences accordingly. 
Unlike existing Hawkes-related processes, our GHP model is able to generate different Hawkes processes with heterogeneous event types but similar dynamics.
For more complicated point processes, we can extend our GHP model by leveraging neural networks and applying multi-dimensional graphons.

Our GHP model is theoretically grounded: 
with mild assumptions, we demonstrate that for the generated Hawkes processes, the proposed model $i$) guarantees their stationarity; $ii$) ensures their parameters to be Lipschitz continuous; and $iii$) makes the difference between their corresponding event sequences bounded.
These properties guarantee the stability of our GHP model when generating Hawkes processes and their event sequences.

Learning GHP from observed heterogeneous event sequences requires us to infer and align the corresponding Hawkes processes with respect to the underlying graphon, for which traditional methods like maximum likelihood estimation are infeasible. 
To overcome this problem, we design a novel learning algorithm based on the reward-augmented maximum likelihood (RAML) estimation~\cite{norouzi2016reward} and the hierarchical optimal transport (HOT) distance~\cite{lee2019hierarchical,yurochkin2019hierarchical}.
In particular, given observed event sequences and those generated by our GHP model, we calculate the HOT distance between them and obtain an optimal transport matrix corresponding to their joint probabilities. 
The probabilities work as the rewards modulating the log-likelihood of each generated event sequence. 
Taking the reward-augment log-likelihood as an objective, we estimate the parameters of GHP  accordingly. 
We verify the feasibility of our GHP model and its learning algorithm on both synthetic and real-world data.
When modeling sparse heterogeneous event sequences that have many event types but small number of events, our GHP model significantly mitigates the risk of over-fitting  and thus outperforms other state-of-the-art point process models. 

\section{Graphon-based Hawkes Processes}\label{sec:model}
\subsection{Generating Hawkes processes from a graphon}\label{ssec:ghp}
For a classic Hawkes process, we often parameterize its impact functions as $\{\phi_{vv'}(t,t')=a_{vv'}\eta(t - t')\}_{v,v'\in\mathcal{V}}$, where the coefficient $a_{vv'}\geq 0$ and the decay kernel $\eta(t)\geq 0$. 
The decay kernel is predefined, and its integral is $D=\int_{0}^{\infty}\eta(t)\text{d}t$. 
Such a Hawkes process is denoted as $\text{HP}_{\mathcal{V}}(\bm{\mu},\bm{A})$, where $\mathcal{V}$ is the set of event types, $\bm{\mu}=[\mu_v]\in \mathbb{R}^{|\mathcal{V}|}$ and $\bm{A}=[a_{vv'}]\in\mathbb{R}^{|\mathcal{V}|\times |\mathcal{V}|}$. 
Here, $|\mathcal{V}|$ is the cardinality of $\mathcal{V}$. 
For $\text{HP}_\mathcal{V}(\bm{\mu},\bm{A})$, $\bm{A}$ is the adjacency matrix of the corresponding Granger causality graph.

A potential way to generate Hawkes processes is to first simulate their Granger causality graphs. 
We apply this strategy based on a nonparametric graph model called \textit{graphon}~\cite{lovasz2012large}. 
A graphon is a two-dimensional measurable function, denoted as $g: \Omega^2\mapsto [0, 1]$, where $\Omega$ is a measure space. 
Given a graphon, we can sample a matrix $\bm{A}=[a_{vv'}]\in [0, 1]^{V\times V}$ with an arbitrary size $V$:
\begin{eqnarray}\label{eq:generate_a}
\begin{aligned}
a_{vv'}=g(x_{v}, x_{v'}),~x_v\sim \text{Uniform}(\Omega)~\text{for}~v=1,..,V.
\end{aligned}
\end{eqnarray}
Here, $\{x_{v}\in\Omega\}_{v=1}^{V}$ are $V$ independent variables sampled from a uniform distribution. 
Accordingly, we generate a graph $G(\mathcal{V},\mathcal{E})$ by setting $\mathcal{V}=\{1,..,V\}$ and $v'\rightarrow v\in\mathcal{E}\sim \text{Bernoulli}(a_{vv'})$. 
This graphon model is fundamental for modeling large-scale networks, which has been widely used in network analysis~\cite{gao2019graphon}. 

Besides $g(x,y)$, we introduce a one-dimensional measurable function on $\Omega$, $i.e.$, $f:\Omega\mapsto [0, +\infty)$, such that we can sample $\bm{\mu}$ and $\bm{A}$ of a Hawkes process from $f(x)$ and $g(x,y)$, respectively. 
Our graphon-based Hawkes process model consists of $f(x)$ and $g(x, y)$, denoted as $\text{GHP}_{\Omega}(f, g)$. 
Here, we set $\Omega=[0, 1]$ and implement the functions as
\begin{eqnarray}\label{eq:functions}
\begin{aligned}
f(x) &= \text{softplus}(f_1)(\exp(\sigma(f_2)x) - 1),\\
g(x,y) &= \sigma\Bigl(\sideset{}{_{i,j\in\{0,..,S\}}}\sum (g_{ij}^1\sin i\pi x + g_{ij}^2\cos i\pi x)\\
&\hspace{2.2cm}\times (g_{ij}^3\sin j\pi y + g_{ij}^4\cos j\pi y) \Bigr),
\end{aligned}
\end{eqnarray}
where $f(x)$ is an exponential function, $g(x,y)$ is designed based on the 2D Fourier series, which has $4(S+1)^2$ coefficients, and $\sigma(\cdot)$ is the sigmoid function. 
This implementation is simple and makes our model satisfy some significant properties in theory, which will be shown in Section~\ref{ssec:stability}.
Then the generative process defined by $\text{GHP}_{\Omega}(f, g)$ is
\begin{eqnarray}\label{eq:ghp}
\begin{aligned}
&\text{HP}_\mathcal{V}(\bm{\mu},\bm{A})\sim \text{GHP}_{\Omega}(f, g):\\
&\quad 1)~V\sim \bm{\pi}=\{\pi_1,...,\pi_{V_{\max}}\},\\
&\quad 2)~\mathcal{V}=\{1,..,V\},~\text{and}~x_v\sim \text{Uniform}(\Omega),~\forall v\in\mathcal{V}.\\
&\quad 3)~\mu_v=f(x_v),~a_{vv'}=\frac{1}{V_{\max}D}g(x_{v}, x_{v'}).\\
&\bm{N}(t)\sim \text{HP}_\mathcal{V}(\bm{\mu},\bm{A}).
\end{aligned}
\end{eqnarray}
Here, $\bm{\pi}$ is a categorical distribution on $\{1, ..., V_{\max}\}$, which is often set as a uniform distribution, and $V_{\max}$ is the maximum number of event types supported by our model. 
We treat $\Omega$ as an uncountable event type space. 
In each trial, we sample $V$ latent event types $\{x_v\}_{v=1}^{V}$ from $\Omega$, where the number of the event types $V$ is sampled from $\bm{\pi}$.
Based on $\{x_v\}_{v=1}^{V}$, we sample $\bm{\mu}$ and $\bm{A}$ from $f$ and $g$, respectively, and instantiate a Hawkes process.
Different from (\ref{eq:generate_a}), we set $a_{vv'}=\frac{1}{V_{\max} D}g(x_{v}, x_{v'})$ in (\ref{eq:ghp}) to ensure the Hawkes process is stationary.
\begin{property}[Stationarity]\label{prop:stationary}
$\text{HP}_\mathcal{V}(\bm{\mu},\bm{A})\sim \text{GHP}_\Omega(f, g)$ is asymptotically stationary as long as $|\mathcal{V}|\leq V_{\max}$. 
\end{property}
Therefore, we can readily generate an event sequence $\bm{N}(t)$ from $\text{HP}_{\mathcal{V}}(\bm{\mu},\bm{A})$ by various simulation methods, $e.g.$, the branch processing~\cite{moller2006approximate} and Ogata's thinning method~\cite{ogata1981lewis}. 

The key challenge in using GHP is that we cannot observe $\{x_v\}_{v=1}^{V}$ because both the event type space $\Omega$ and the sampled event types are latent.  
Accordingly, for the generated Hawkes processes and their event sequences, we cannot directly match their event types ($i.e.$, $\{x_v\}_{v=1}^{V}$) with the event types of real-world sequences. 
To solve this problem, in Section~\ref{ssec:align} we will leverage optimal transport~\cite{villani2008optimal,peyre2019computational} to measure the distance between heterogeneous event sequences. 
The learned optimal transport helps us to find a {\it soft alignment} between the generated event types and the real ones, which not only makes the generated event types and the corresponding point processes semantically meaningful but also builds the foundation for the learning method of our model (See Section~\ref{sec:learn}). 

\subsection{Extensions}\label{ssec:implementation}
The proposed GHP provides us with a new framework to jointly model heterogeneous event sequences. 
Beyond Hawkes processes, our GHP model can be readily extended to generate more sophisticated types of point processes.

\textbf{Nonlinear Hawkes process}. For nonlinear Hawkes process (also called mutually-correcting process)~\cite{zhu2013nonlinear,xu2016patient}, its intensity function is $\lambda_v(t)=\exp(\mu_v + \sum_{t_i<t}\phi_{vv_i}(t,t_i))$ and the parameters can be negative. 
In this case, we can implement $\text{GHP}_{\Omega}(f, g)$ with $f:\Omega\mapsto(-\infty,+\infty)$ and $g:\Omega^2\mapsto(-\infty,+\infty)$, respectively.

\textbf{Multi-kernel Hawkes process}. The multi-kernel Hawkes process constructs its impact functions by a set of decay kernels~\cite{xu2016learning}, $i.e.$, $\phi_{vv'}(t)=\sum_{m=1}^{M}a_{vv'm}\eta_{m}(t)$, where the coefficients $a_{vv'm}$'s are formulated as $M$ matrices $\{\bm{A}_m\}_{m=1}^{M}$.  
In this case, we need to introduce several graphons, $i.e.$, $\{g_1(x,y),...,g_{M}(x,y)\}$, to generate the $M$ matrices, and our GHP model becomes $\text{GHP}_{\Omega}(f, \{g_m\}_{m=1}^{M})$.

\textbf{Time-varying Hawkes process}. The time-varying Hawkes process applies shift-varying impact functions, $i.e.$, $\phi_{vv'}(t,t')=a_{vv'}(t)\eta(t-t')$, where the coefficient $a_{vv'}(t)$ becomes a function of time. 
Similar to the multi-kernel Hawkes process, when using a set of bases to represent the coefficient function~\cite{xu2017learning}, $i.e.$, $a_{vv'}(t)=\sum_{m=1}^{M}a_{vv'm}h_{m}(t)$, where $h_m(t)$ is the $m$-th base, we can still apply multiple graphons to generate impact functions and rewrite our GHP model as $\text{GHP}_{\Omega}(f, \{g_m\}_{m=1}^{M})$.

\textbf{Neural Hawkes process} Most existing neural network-based Hawkes processes apply embedding layers to map the index of each event type to its latent code~\cite{mei2017neural,zhang2020self,zuo2020transformer}.
For the neural Hawkes process, we can replace the embedding layer with a function $f(x):\Omega\mapsto \mathbb{R}^M$ such that we can generate $M$-dimensional latent codes for uncountable event types in $\Omega$. 
If the neural Hawkes process considers the interactions of different event types~\cite{wang2016coevolutionary}, we can set the graphon as $g(x,y)=p(x)^{\top}q(y)$, where $p(x):\Omega\mapsto \mathbb{R}^M$ and $q(y):\Omega\mapsto \mathbb{R}^M$, respectively. 
Accordingly, the GHP becomes $\text{GHP}_{\Omega}(f, p, q)$. 
Besides changing the point process model, we can also implement $f(x)$ and $g(x,y)$ by deep neural networks, which is left for future work.

\subsection{Theoretical analysis of the GHP model}\label{ssec:stability}
In addition to verifying the stationarity of generated Hawkes processes, we demonstrate two more properties of GHP  based on the following mild assumptions.
\begin{assumption}\label{assum:lip}
For $\text{GHP}_{\Omega}(f, g)$, we assume\vspace{-10pt}
\begin{itemize}
    \item[A)]$f(x)$ is bi-Lipschitz continuous on $\Omega$, denoted as $f\in \text{Lip}_{\Omega}(C_{1}^f,C_{2}^f)$: $\exists~0<C_1^f\leq C_2^f<\infty$, $C_{1}^f\|x-x'\|_2\leq |f(x)-f(x')|\leq C_{2}^f\|x-x'\|_2$, $\forall x,x'\in\Omega$.
    \vspace{-5pt}
    \item[B)]$f(x)$ has a unique zero point in $\Omega$, $i.e.$, $f(x_0^f)=0$.
    \vspace{-5pt}
    \item[C)]$g(x, y)$ is strictly smaller than $1$, $i.e.$, $g:\Omega^2\mapsto [0, 1)$.\vspace{-5pt}
    \item[D)]$g(x,y)$ is Lipschitz continuous on $\Omega^2$, denoted as $g\in\text{Lip}_{\Omega^2}(C^g)$: $\exists~0<C^g<\infty$, $|g(x,y)-g(x',y')|\leq C^g\|[x;y]-[x';y']\|_2$, $\forall [x;y],[x';y']\in\Omega^2$.
\end{itemize}
\end{assumption}
\vspace{-5pt}
Clearly, GHP defined in (\ref{eq:functions}) satisfies the assumptions. 
Based on the above assumptions, we prove that the parameters of the Hawkes process generated by our GHP model is Lipschitz continuous. 
\begin{property}[Lipschitz Continuity]\label{prop:lip}
For $\text{HP}_{\mathcal{V}}(\bm{\mu}_1,\bm{A}_1)$ and $\text{HP}_{\mathcal{U}}(\bm{\mu}_2,\bm{A}_2)\sim \text{GHP}_{\Omega}(f, g)$, where $\text{GHP}_{\Omega}(f, g)$ satisfies Assumption~\ref{assum:lip}, their parameters satisfy
\begin{eqnarray}\label{eq:lip}
\begin{aligned}
&C_{1}^f d_{\text{w}}(\bm{x}_1,\bm{x}_2)\leq d_{\text{w}}(\bm{\mu}_1,\bm{\mu}_2)\leq C_{2}^f d_{\text{w}}(\bm{x}_1,\bm{x}_2),\\
&d_{\text{w}}(\bm{A}_1,\bm{A}_2)\leq C^g d_{\text{w}}(\bm{x}_1^{\times},\bm{x}_2^{\times}),\\
&d_{\text{gw}}(\bm{A}_1,\bm{A}_2)\leq C^g d_{\text{gw}}(\bm{x}_1^{\times},\bm{x}_2^{\times}),
\end{aligned}
\end{eqnarray}
where $\bm{x}_1=\{x_{v,1}\}_{v=1}^{|\mathcal{V}|}$ and $\bm{x}_2=\{x_{u,2}\}_{u=1}^{|\mathcal{U}|}$ are the latent event types, and $\bm{x}_1^{\times}=\{[x_{v,1};x_{v',1}]\}_{v,v'=1}^{|\mathcal{V}|}$ and $\bm{x}_2^{\times}=\{[x_{u,2};x_{u',2}]\}_{u,u'=1}^{|\mathcal{U}|}$ enumerate the pairs of the latent event types.
$d_{\text{w}}$ is the discrete Wasserstein distance (or called the earth mover's distance) and the $d_{\text{gw}}$ is the discrete Gromov-Wasserstein distance.\footnote{The definitions of $d_{\text{w}}$ and $d_{\text{gw}}$ are given in Appendix A.}
\end{property}
Property~\ref{prop:lip} shows that $i$) for the generated Hawkes processes, the difference between their parameters is bounded by the difference between their latent event types; and $ii$) the parameters of each generated Hawkes process are robust to the perturbations of the latent event types. 

Because the difference between generated Hawkes processes is bounded, the difference between the corresponding event sequences is bounded as well.
Specifically, for a point process, its average intensity vector, defined as $\bar{\bm{\lambda}} := \frac{\mathbb{E}[\text{d}\bm{N}(t)]}{\text{d}t}$, reflects the dynamics of its event sequences~\cite{chiu2013stochastic}. 
For this key statistics, we have 
\begin{property}\label{prop:error}
For $\text{HP}_{\mathcal{V}}(\bm{\mu}_1,\bm{A}_1)$ and $\text{HP}_{\mathcal{U}}(\bm{\mu}_2,\bm{A}_2)\sim \text{GHP}_{\Omega}(f, g)$, where $\text{GHP}_{\Omega}(f, g)$ satisfies Assumption~\ref{assum:lip} and $|\mathcal{V}|\leq |\mathcal{U}|$, their average intensity vectors, $i.e.$, $\bar{\bm{\lambda}}_1$ and $\bar{\bm{\lambda}}_2$, satisfy
\begin{eqnarray}\label{eq:error}
\begin{aligned}
\frac{d_{\text{w}}(\bar{\bm{\lambda}}_1,\bar{\bm{\lambda}}_2)}{\|\bar{\bm{\lambda}}_1\|_2}
\leq& \frac{\frac{\sqrt{2U}C^g}{C_1^f\|\bm{I}_V-D\bm{A}_1\|_2}+\frac{1}{\|\bm{\mu}_1\|_2}}{1-D\|\bm{A}_1\|_2}\Bigl(d_{\text{w}}(\bm{\mu}_1, \bm{\mu}_2)\\
&+\sqrt{\frac{U-V}{V}}\|\bm{\mu}_1\|_2\Bigr)+\sqrt{\frac{U-V}{UV}},
\end{aligned}
\end{eqnarray}
where $\|\cdot\|_2$ is the $\ell_2$-norm for vectors and the spectral norm for matrices, $U=|\mathcal{U}|$, $V=|\mathcal{V}|$, $D=\int_{0}^{\infty} \eta(t)\text{d}t$ is the integral of the decay kernel used in the Hawkes processes, and $C_1^f$ and $C^g$ are the constants defined in Assumption~\ref{assum:lip}.
\end{property}
Furthermore, if $|\mathcal{V}|=|\mathcal{U}|$, we can simplify Property~\ref{prop:error} as 
\begin{corollary}\label{coro:error_s}
For $\text{HP}_{\mathcal{V}}(\bm{\mu}_1,\bm{A}_1)$ and $\text{HP}_{\mathcal{U}}(\bm{\mu}_2,\bm{A}_2)\sim \text{GHP}_{\Omega}(f, g)$, where $\text{GHP}_{\Omega}(f, g)$ satisfies Assumption~\ref{assum:lip} and $|\mathcal{V}|=|\mathcal{U}|=V$, we have
\begin{eqnarray*}
\begin{aligned}
\frac{d_{\text{w}}(\bar{\bm{\lambda}}_1,\bar{\bm{\lambda}}_2)}{\|\bar{\bm{\lambda}}_1\|_2}
\leq\frac{d_{\text{w}}(\bm{\mu}_1, \bm{\mu}_2)}{1-D\|\bm{A}_1\|_2}\left(\frac{\sqrt{2V}{C^g}/{C_1^f}}{\|\bm{I}_V-D\bm{A}_1\|_2}+\frac{1}{\|\bm{\mu}_1\|_2}\right).
\end{aligned}
\end{eqnarray*}
\end{corollary}

\section{Learning Algorithm}\label{sec:learn}
\subsection{A reward-augmented maximum likelihood}
We propose a novel method to learn GHP model from observed heterogeneous event sequences.
Denote $\mathcal{N}=\{\bm{N}_l(t)\}_{l=1}^{L}$ as the set of real-world event sequences and $\widehat{\mathcal{N}}=\{\widehat{\bm{N}}_k(t)\}_{k=1}^{K}$ the set of the event sequences generated by our model. 
Because the correspondence of real-world event types in the latent event type space is unknown, as mentioned in Section~\ref{ssec:ghp}, we need to {\it simultaneously} learn the underlying graphon of our model and align the event types of the generated Hawkes processes with the real ones. 
To achieve this aim, we formulate the following optimization problem,
\begin{eqnarray}\label{eq:raml2}
\begin{aligned}
\sideset{}{_{\theta}}\min -\sideset{}{_{\widehat{\bm{N}}_k\in\widehat{\mathcal{N}}}}\sum \sideset{}{_{\bm{N}_l\in\mathcal{N}}}\max q(\widehat{\bm{N}}_k| \bm{N}_l)\log p(\widehat{\bm{N}}_k;\theta).
\end{aligned}
\end{eqnarray}
where $p(\widehat{\bm{N}}_k;\theta) = \frac{\prod_{(t_i, v_i)}\lambda_{v_i}^k(t_i;\theta)}{\exp(\sum_{v\in\mathcal{V}}\int_{0}^{T}\lambda_v^k(t;\theta)\text{d}t)}$ is the likelihood of the $k$-th generated event sequence, $\theta$ represents the model parameter $\{f_1, f_2, \{g_{ij}^{m}\}\}$, and $q(\widehat{\bm{N}}_k|\bm{N}_l)$ is the probability of $\widehat{\bm{N}}_k$ conditioned on the $l$-th real sequence $\bm{N}_l$. 
Essentially, the conditional probability $q(\widehat{\bm{N}}|\bm{N})$ measures the similarity between the generated sequence and the real one. 
When the two sequences yield the same generative mechanism and have similar dynamics, the real sequence provides useful prior information, and thus, the occurrence of the generated sequence is with a high probability. 

In (\ref{eq:raml2}), the log-likelihood of each generated sequence is weighted by its maximum conditional probability with respect to the real sequences, $i.e.$, $\max_{\bm{N}_l\in\mathcal{N}} q(\widehat{\bm{N}}_k | \bm{N}_l))$. 
The weight measures the overall similarity between the the generated sequence $\widehat{\bm{N}}_k$ and the real ones. 
A large weight indicates that the generated sequence is informative for our learning problem because it is similar to at least one real sequence. 
Otherwise, the sequence is less useful. 
Additionally, assuming the empirical distribution of the real sequences to be uniform, we have $q(\widehat{\bm{N}}|\bm{N})\propto q(\widehat{\bm{N}},\bm{N})$, and the optimization problem becomes,
\begin{eqnarray}\label{eq:raml3}
\begin{aligned}
\sideset{}{_{\theta}}\min -\sideset{}{_{\widehat{\bm{N}}_k\in\widehat{\mathcal{N}}}}\sum \sideset{}{_{\bm{N}_l\in\mathcal{N}}}\max q(\widehat{\bm{N}}_k, \bm{N}_l)\log p(\widehat{\bm{N}}_k;\theta).
\end{aligned}
\end{eqnarray}
{\sc Remark.} the above formulation (\ref{eq:raml3}) can be considered as a variant of the reward-augmented maximum likelihood (RAML) estimation method~\cite{norouzi2016reward} (see also Section~\ref{ssec:analysis}). 
For sequence $\widehat{\bm{N}}_k$, the weight $\max_{\bm{N}_l\in\mathcal{N}} q(\widehat{\bm{N}}_k, \bm{N}_l)$ plays the role of its \textit{reward} and is assigned to its log-likelihood. 
The higher reward the log-likelihood obtains, the more significant it is in learning. 

\subsection{Hierarchical optimal transport between heterogeneous event sequences}\label{ssec:align}
The key of our learning algorithm, which is also its main novelty, is computing the joint distribution $q(\widehat{\bm{N}},\bm{N})$ based on the hierarchical optimal transport (HOT) model~\cite{lee2019hierarchical,yurochkin2019hierarchical}. 
In particular, the HOT model not only captures the optimal transport between the generated event sequences and the real ones but also captures the optimal transport between their event types. 
Given $\widehat{\mathcal{N}}=\{\widehat{\bm{N}}_{k}\}_{k=1}^{K}$ and $\mathcal{N}=\{\bm{N}_{l}\}_{l=1}^{L}$, we compute the optimal transport distance between them as $d_{\text{ot}}(\widehat{\mathcal{N}}, \mathcal{N})$
\begin{eqnarray}\label{eq:ot_tpp2}
\begin{aligned}
:=&\sideset{}{_{\bm{Q}\in\Pi\left(\frac{1}{K}\bm{1}_K, \frac{1}{L}\bm{1}_L\right)}}\min \sideset{}{_{k,l}}\sum q(\widehat{\bm{N}}_k,\bm{N}_l)d(\widehat{\bm{N}}_{k}, \bm{N}_{l})\\
=&\sideset{}{_{\bm{Q}\in\Pi\left(\frac{1}{K}\bm{1}_K, \frac{1}{L}\bm{1}_L\right)}}\min \langle\bm{D},\bm{Q}\rangle,
\end{aligned}
\end{eqnarray}
where the polytope $\Pi\left(\frac{1}{K}\bm{1}_K, \frac{1}{L}\bm{1}_L\right)=\{\bm{Q}\geq\bm{0}~|~\bm{Q}\bm{1}_{L}=\frac{1}{K}\bm{1}_K,\bm{Q}^{\top}\bm{1}_{K}=\frac{1}{L}\bm{1}_L\}$ is the set of the doubly-stochastic matrices having marginals $\frac{1}{K}\bm{1}_K$ and $\frac{1}{L}\bm{1}_L$, $\bm{D}=[d(\widehat{\bm{N}}_{k}, \bm{N}_{l})]\in\mathbb{R}^{K \times L}$ is a distance matrix, whose element measures the distance between the sequences. 

$\bm{Q}^*=\arg\min_{\bm{Q}\in\Pi(\frac{1}{K}\bm{1}_K, \frac{1}{L}\bm{1}_L)}\langle\bm{D},\bm{Q}\rangle$,
the optimizer of (\ref{eq:ot_tpp2}), is the optimal transport matrix between the two sets of event sequences. 
When $\widehat{\mathcal{N}}$ and $\mathcal{N}$ correspond to the sets of generated event sequences and the real ones, this matrix is the desired joint distribution, $i.e.$, $\bm{Q}^*=[q^*(\widehat{\bm{N}}_k,\bm{N}_l)]$. 
This optimization problem can be solved by many efficient methods, $e.g.$, the Sinkhorn scaling method~\cite{cuturi2013sinkhorn} and the proximal point method~\cite{xie2020fast}.

\begin{figure}[t!]
    \centering
    \includegraphics[width=0.85\linewidth]{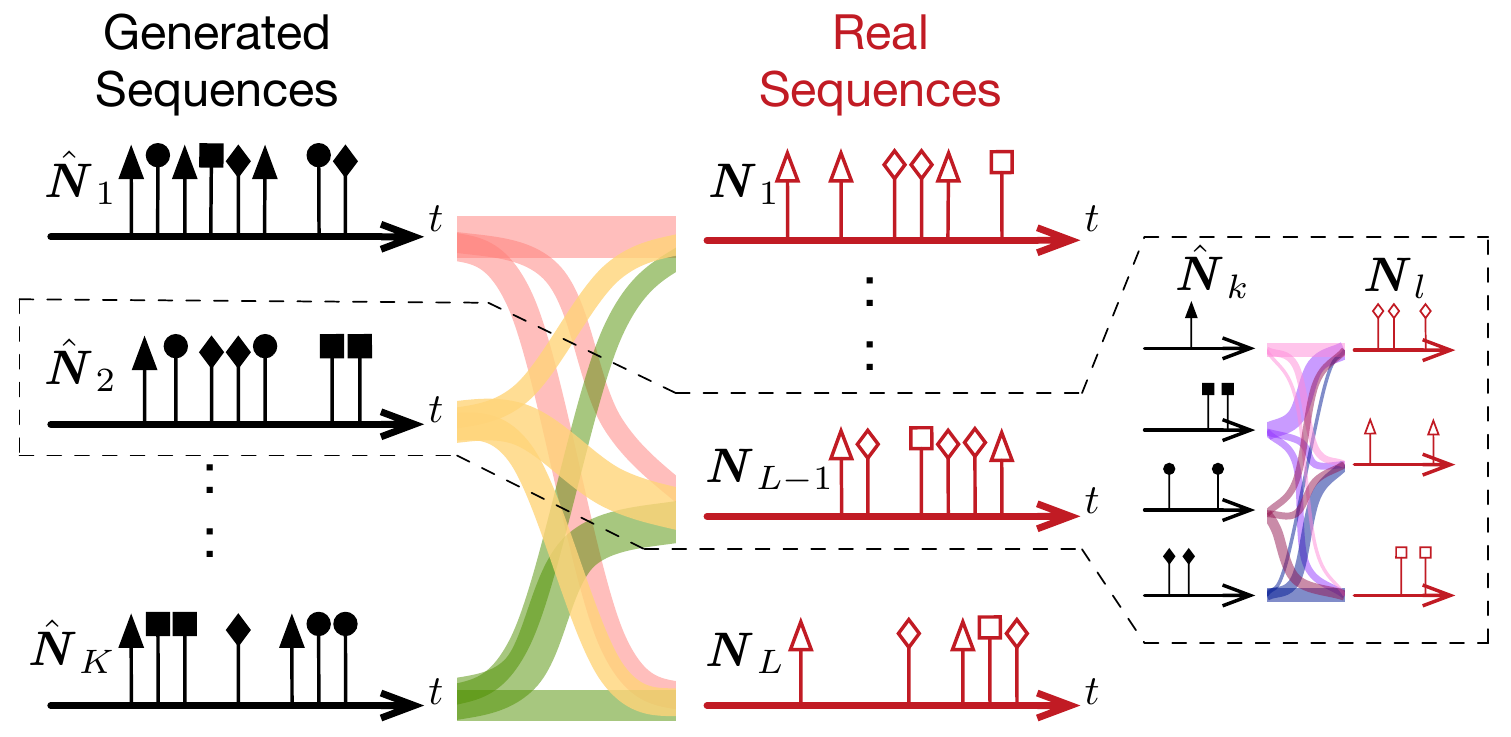}
    \vspace{-4mm}
    \caption{\small{An illustration of the hierarchical optimal transport distance between two sets of event sequences.}}
    \label{fig:hot_seq}
\end{figure}

For $\widehat{\bm{N}}_k=\{\widehat{N}_{u}^{k}\}_{u\in\mathcal{V}_k}$ and $\bm{N}_l=\{N_{v}^{l}\}_{v\in\mathcal{V}_l}$, where $\mathcal{V}_k$ and $\mathcal{V}_l$ are the sets of their event types, we also implement their distance 
as an optimal transport distance: $d(\widehat{\bm{N}}_k,\bm{N}_l)$
\begin{eqnarray}\label{eq:ot_cp}
\begin{aligned}
:=&\sideset{}{_{\bm{T}\in\Pi\left(\frac{1}{|\mathcal{V}_k|}\bm{1}_{|\mathcal{V}_k|}, \frac{1}{|\mathcal{V}_l|}\bm{1}_{|\mathcal{V}_l|}\right)}}\min \sideset{}{_{u,v}}\sum T_{uv}d(\widehat{N}_{u}^k, N_{v}^l)\\
=&\sideset{}{_{\bm{T}\in\Pi\left(\frac{1}{|\mathcal{V}_k|}\bm{1}_{|\mathcal{V}_k|}, \frac{1}{|\mathcal{V}_l|}\bm{1}_{|\mathcal{V}_l|}\right)}}\min \langle\bm{D}_{kl},\bm{T}\rangle,
\end{aligned}
\end{eqnarray}
where $\bm{D}_{kl}=[d(\widehat{N}_{u}^k, N_{v}^l)]\in\mathbb{R}^{|\mathcal{V}_k|\times |\mathcal{V}_l|}$ is the distance matrix for $\widehat{\bm{N}}_k$ and $\bm{N}_l$, and $d(\widehat{N}_{u}^k, N_{v}^l) = \frac{1}{T}\int_{0}^{T}|\widehat{N}_u^k(t) - N_v^l(t)|\text{d}t$ measures the difference between the sequence of the type-$u$ events and that of the type-$v$ events in $[0, T]$. 

Plugging (\ref{eq:ot_cp}) into (\ref{eq:ot_tpp2}), we measure the difference between two sets of heterogeneous event sequences by a hierarchical optimal transport distance, in which the ground distance used in (\ref{eq:ot_tpp2}) is also an optimal transport distance.
Figure~\ref{fig:hot_seq} illustrates the hierarchical optimal transport distance.
In the proposed HOT distance, the optimal transport matrix $\bm{Q}^*$ derived by (\ref{eq:ot_tpp2}) achieves a {\it soft} alignment between the generated sequences and the real ones, which corresponds to the joint distribution in (\ref{eq:raml3}). 
Additionally, the optimal transport matrix $\bm{T}^*$ derived by (\ref{eq:ot_cp}) aligns the event types of a generated sequence with those of a real one, which indicates the correspondence of real-world event types in the latent event type space. 
In Section~\ref{sec:exp}, we will show that based on $\bm{T}^*$  GHP can generate semantically-meaningful Hawkes processes and their event sequences.

\subsection{Further analysis}\label{ssec:analysis} 
Our HOT-based RAML method (denoted as RAML-HOT) has two advantages over the original RAML in~\cite{norouzi2016reward}. 
Firstly, the reward used in the original RAML is the sum of the conditional probabilities, $i.e.$, $\sum_{\bm{N}_l\in\mathcal{N}} q(\hat{\bm{N}}_k | \bm{N}_l)$. 
Accordingly, a generated sequence earns a high reward only when it is similar to most real sequences. 
This setting is unnecessary even unreasonable in our problem: a generated sequence is likely to close to a small number of real sequences because the real sequences are heterogeneous and yield different Hawkes processes. 
Secondly, the original RAML implements $q(\widehat{\bm{N}} | \bm{N})$ empirically as an exponential pay-off distribution, $i.e.$,  
$q(\widehat{\bm{N}} | \bm{N}) = \frac{1}{Z}\exp(\frac{r(\widehat{\bm{N}},\bm{N})}{\tau})$, where $Z$ is the normalizing constant, $\tau$ is the hyperparameter, and $r(\widehat{\bm{N}},\bm{N})$ is a predefined reward function. 
Different from the original RAML, our RAML-HOT method computes the joint distribution $q(\widehat{\bm{N}},\bm{N})$ based on the HOT distance and the reward $\max_{\bm{N}_l\in\mathcal{N}} q(\widehat{\bm{N}}_k, \bm{N}_l)$ is more reasonable and interpretable. 

Algorithm~\ref{alg:pg} shows the steps of our learning method and the original RAML when learning a GHP model.
Given $L$ real sequences, for each of them we denote $\mathcal{O}(V)$ as the number of its event types and $\mathcal{O}(I)$ the number of events per event type. 
When learning GHP, RAML-HOT generates a batch of sequences and computes its HOT distance to a batch of real sequences.
Because of solving $B^2 + 1$ optimal transport problems, its computational complexity is $\mathcal{O}(B^2IV^2)$, where $B$ is batch size. 
Regarding computational cost, GHP is suitable for modeling {\it sparse} heterogeneous event sequences, in which both $V$ and $I$ are small and thus our RAML-HOT method is efficient. 
Such sequences are {\it common} in real-world applications: $i$) The admissions of different patients in a hospital cover many kinds of diseases, but each patient often has a limited number of diseases and admissions. $ii$) The Linkedin users cover many types of jobs, but each user has few job-hopping behaviors among a small number of jobs. 
In such situations, GHP captures the point process per sequence, whose number of event types ($i.e.$, $V$) is limited. 
Compared to modeling a large point process model for all the sequences, applying our GHP model can mitigate the risk of over-fitting. 

\section{Related Work}
\textbf{Hawkes processes}.
Because of its quantitative power and good interpretability, Hawkes process has been a significant tool for event sequence analysis and achieved encouraging performance in many applications like social network analysis~\cite{zhou2013learning,farajtabar2017coevolve} and financial engineering~\cite{bacry2015hawkes}. 
These years, many efforts have been made to develop the variants of Hawkes process, $e.g.$, the mixture model of Hawkes processes~\cite{xu2017dirichlet}, the recurrent neural networks in the continuous time~\cite{du2016recurrent,mei2017neural} and the Hawkes processes with attention mechanisms~\cite{zhang2020self,zuo2020transformer}. 
Most existing models are learned by the maximum likelihood estimation. 
Recently, more cutting-edge techniques are applied, $e.g.$, Wasserstein generative adversarial network~\cite{xiao2017wasserstein}, reinforcement learning~\cite{li2018learning}, and noisy contrastive estimation~\cite{mei2020noise}. 
However, most existing methods cannot learn multiple Hawkes processes with different event types.

\begin{algorithm}[t]
\small{
	\caption{Learning a GHP model}
	\label{alg:pg}
	\begin{algorithmic}[1]
	    \STATE \textbf{Input} Real event sequences $\mathcal{N}$.
	    \STATE Initialize the model parameter $\theta$ randomly.
	    \STATE \textbf{for} each epoch
	    \STATE \quad\textbf{for} each batch of real sequences $\{\bm{N}_b\}_{b=1}^{B}\subset \mathcal{N}$
	    \STATE \quad\quad Generate $B$ sequences $\{\widehat{\bm{N}}_b\}_{b=1}^{B}$ via (\ref{eq:ghp}).
	    \STATE \quad\quad Calculate $d(\widehat{\bm{N}}_b,\bm{N}_{b'})$ by (\ref{eq:ot_cp}) and obtain the matrix $\bm{D}$.
	    \STATE \quad\quad \textbf{RAML:}\\
	    \quad\quad\quad Set the reward function $r(\widehat{\bm{N}}_b,\bm{N}_{b'})=-d(\widehat{\bm{N}}_b,\bm{N}_{b'})$\\ 
	    \quad\quad\quad and $q(\widehat{\bm{N}}_b|\bm{N}_{b'})$ an exponential pay-off distribution.
	    \STATE \quad\quad \textbf{Our RAML-HOT:}\\ 
	    \quad\quad\quad Solve (\ref{eq:ot_tpp2}) and obtain $\bm{Q}^*=[q^*(\widehat{\bm{N}}_b,\bm{N}_{b'})]$.
	    \STATE \quad\quad Calculate the loss function in (\ref{eq:raml3}).
	    \STATE \quad\quad Update $\theta$ by the Adam algorithm~\cite{kingma2014adam}.
	\end{algorithmic}
}
\end{algorithm}

\textbf{Graphons}.
Graphon is a nonparametric graph model generating arbitrary-size graphs in an infinite dimensional space~\cite{lovasz2012large}. 
Given observed graphs, most existing methods learn graphons as stochastic block models~\cite{channarond2012classification,airoldi2013stochastic,chan2014consistent}, low-rank matrices~\cite{keshavan2010matrix,chatterjee2015matrix,xu2018rates} or Gromov-Wasserstein barycenters~\cite{xu2021learning}, which approximate graphons by 2D step functions based on the weak regularity lemma~\cite{frieze1999quick}. 

\textbf{Optimal transport}.
The theory of optimal transport~\cite{villani2008optimal} has been widely used in distribution estimation~\cite{boissard2015distribution} and matching~\cite{courty2017learning}, and data generation~\cite{arjovsky2017wasserstein}. 
Because of its usefulness, many methods have been proposed to compute the optimal transport efficiently, $e.g.$, the Sinkhorn scaling algorithm~\cite{cuturi2013sinkhorn} and its stochastic variant~\cite{altschuler2017near}, the Bregman ADMM algorithm~\cite{wang2014bregman}, the proximal point method~\cite{xie2020fast}, and the sliced Wasserstein distance~\cite{kolouri2018sliced}.
Recently, hierarchical optimal transport (HOT) models are proposed in~\cite{lee2019hierarchical,yurochkin2019hierarchical}, which achieve encouraging performance on data clustering. 
Our work makes the first attempt to introduce the HOT model into event sequence analysis.

\begin{figure*}[t]
    \centering
    \subfigure[$d_{\text{fgw}}(\hat{\theta},\theta)$]{
    \includegraphics[height=3.5cm]{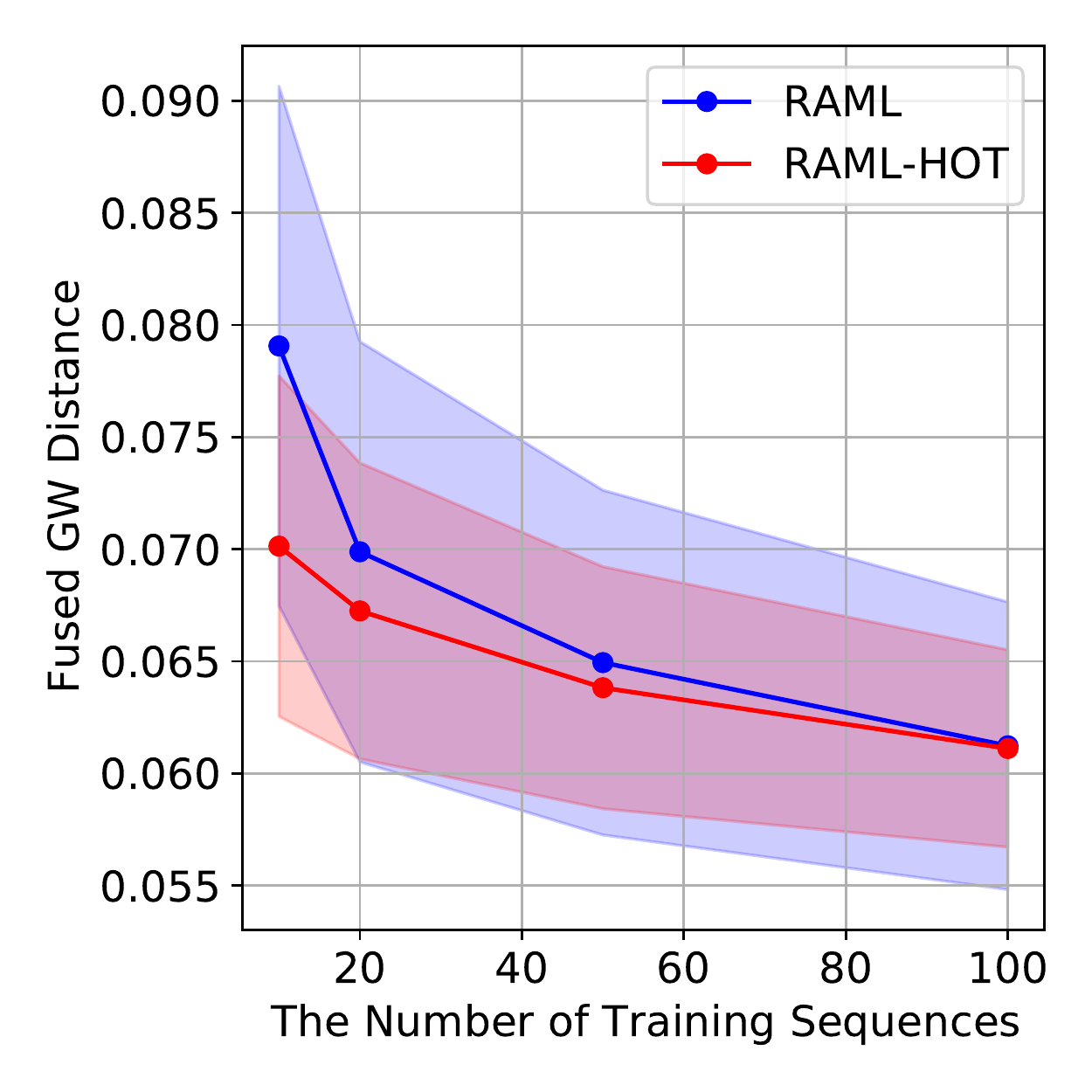}\label{fig:fgwd}
    }
    \subfigure[$d_{\text{ot}}(\widehat{\mathcal{N}},\mathcal{N})$]{
    \includegraphics[height=3.5cm]{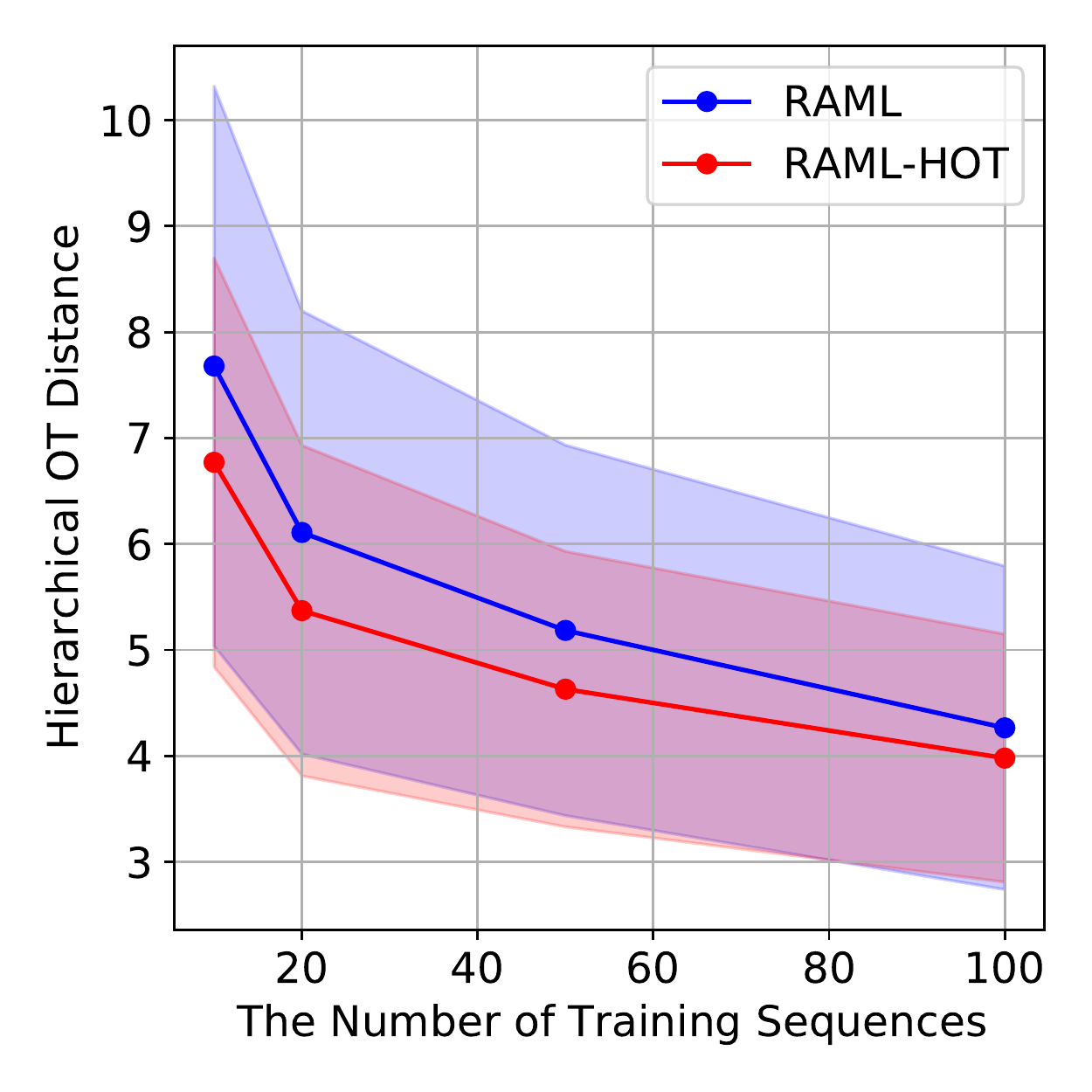}\label{fig:hot}
    }
    \subfigure[The influence of $B$]{
    \includegraphics[height=3.5cm]{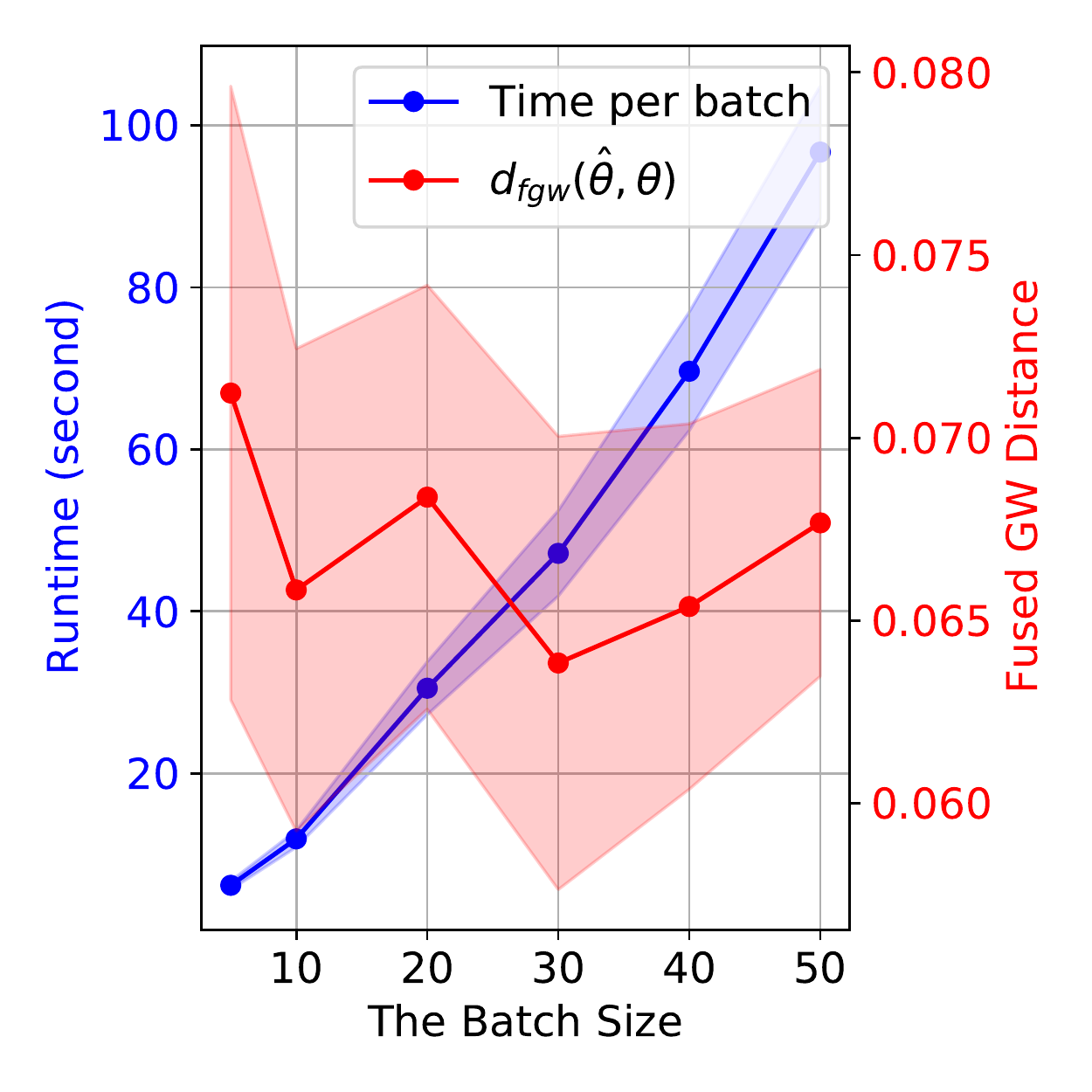}\label{fig:batch}
    }
    \subfigure[The influence of $V_{\max}$]{
    \includegraphics[height=3.5cm]{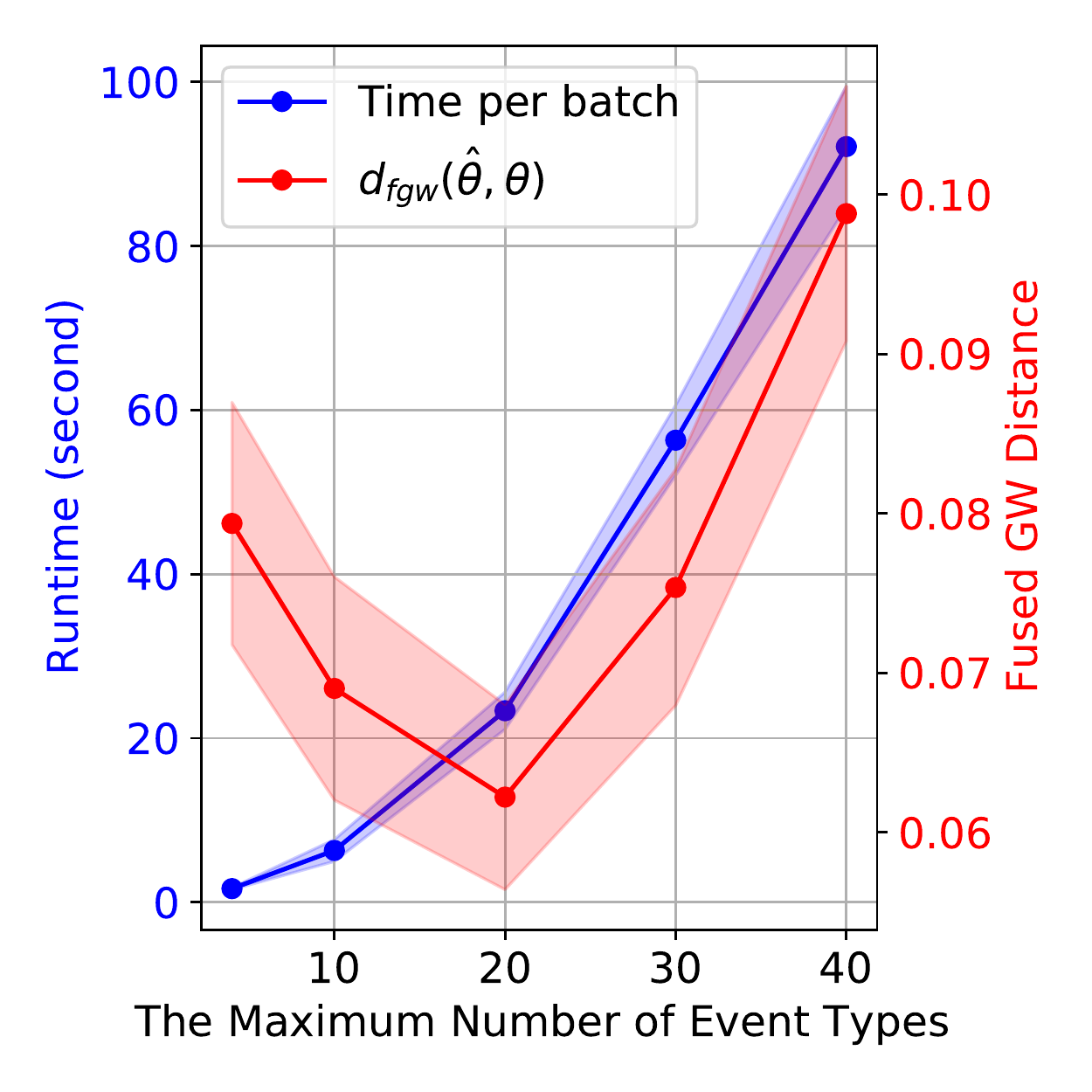}\label{fig:number}
    }
    \vspace{-4mm}
    \caption{The experimental results on synthetic data. The runtime in (c, d) is derived by running on a CPU.}
    \label{fig:syn1}
\end{figure*}

\section{Experiments}\label{sec:exp}
\subsection{Experiments on synthetic data}
To test our learning method, we first learn GHP models from synthetic heterogeneous event sequences. 
The synthetic sequences are generated by a predefined GHP model.
For the predefined model, we set $V_{\max}=20$, the decay kernel $\kappa(t)=\exp(-t)$, the number of Fourier bases ($i.e.$, the $S$ in (\ref{eq:functions})) of $g(x,y)$ as $5$, and sampled the model parameter $\theta$ from a multivariate normal distribution.
Given the predefined model, we simulate $120$ event sequences in the time window $[0, 50]$ by the steps in (\ref{eq:ghp}), in which we apply the Ogata's thinning method~\cite{ogata1981lewis}.
We select $100$ sequences for learning a new GHP model, $10$ sequences for validation, and the remaining $10$ sequences for testing the learned model. 
We evaluate the learned model based on two criteria.
Firstly, we compute the Fused Gromov-Wasseserstein (FGW) distance~\cite{vayer2018fused} between the estimated model parameter $\hat{\theta}$ and the ground truth $\theta$:
\begin{eqnarray}
\begin{aligned}
d_{\text{fgw}}(\hat{\theta},\theta):=&\sideset{}{_{\pi\in\Pi}}\inf\mathbb{E}_{x,x'\sim \pi}[|\hat{f}(x)-f(x')|^2] + \\
&\mathbb{E}_{x,x',y,y'\sim \pi\otimes\pi}[|\hat{g}(x,y)-g(x',y')|^2].
\end{aligned}
\end{eqnarray}
The FGW distance minimizes the expected error between the model parameters by finding finds an optimal transport $\pi$, whose implementation is in Appendix B. 
Secondly, we simulate a set of sequences based on the learned model and calculate its HOT distance to the testing set, $i.e.$, $d_{\text{ot}}(\widehat{\mathcal{N}},\mathcal{N})$. 

Setting the number of training sequences from $10$ to $100$, we test our learning method (\textbf{RAML-HOT}) and compare it with the original \textbf{RAML}~\cite{norouzi2016reward}. 
For each method, we set the number of epochs to be $20$ and the learning rate to be $0.01$. 
For our RAML-HOT method, we apply the Sinkhorn scaling method~\cite{cuturi2013sinkhorn} to compute the HOT distance. 
Figure~\ref{fig:fgwd} and Figure~\ref{fig:hot} show the averaged performance of the two learning methods in $10$ trials. 
With the increase of training data, both our RAML-HOT and the RAML improve their learning results consistently, achieving smaller $d_{\text{fgw}}(\hat{\theta},\theta)$ and $d_{\text{ot}}(\widehat{\mathcal{N}},\mathcal{N})$ with smaller standard deviation. 
Moreover, we can find that our RAML-HOT method outperforms the RAML method on the two measurements. 
This result verifies the feasibility of our RAML-HOT method and demonstrates its advantages claimed in Section~\ref{ssec:analysis} --- the reward used in (\ref{eq:raml2}) is suitable for our problem, and leveraging the HOT distance works better than using the exponential pay-off distribution. 

For our RAML-HOT method, the batch size $B$ is a key hyperparameter. 
Generally, using a large batch size may improve learning results.
However, for our method, whose computational complexity is quadratic to the batch size, we need to carefully set the batch size to achieve a trade-off between performance and efficiency.
Figure~\ref{fig:batch} visualize the runtime per batch and the  $d_{\text{fgw}}(\hat{\theta},\theta)$ achieved by our method with respect to different batch sizes. 
We find that the $d_{\text{fgw}}(\hat{\theta},\theta)$ is relatively stable but the runtime increases quadratically with respect to the batch size. 
According to the result, we set $B=10$ in our experiments.

Besides the batch size, the maximum number of event types $V_{\max}$ is also significant. 
According to (\ref{eq:ghp}), for the event sequences generated by our GHP model, the expected number of their event types is $\mathbb{E}[V]=\frac{\hat{V}_{\max}}{2}$. 
In the training phase, the maximum number of event types used to learn the GHP model, denoted as $\hat{V}_{\max}$, may be different from the ground truth $V_{\max}$. 
Setting $\hat{V}_{\max}$ too large or too small may lead to the model misspecification problem.
As shown in Figure~\ref{fig:number}, the runtime of our method increases quadratically with respect to $\hat{V}_{\max}$, which verifies the computational complexity in Section~\ref{ssec:analysis}.
The best $d_{\text{fgw}}(\hat{\theta},\theta)$ is achieved when the $\hat{V}_{\max}=V_{\max}$. 
In practice, given a set of training sequences, we calculate the averaged number of event types per sequence, denoted as $\bar{V}$, and set $\hat{V}_{\max}=2\bar{V}$.

\subsection{Modeling sparse heterogeneous event sequences}
As aforementioned, our GHP model is suitable for modeling sparse heterogeneous event sequences. 
We demonstrate the usefulness of our GHP model on two representative real-world datasets. 
The first is the Linkedin dataset, which contains the job-hopping and promotion behaviors of 2,439 Linkedin users~\cite{xu2017learning}. 
The dataset has 3,730 kinds of jobs ($i.e.$, the event types).
However, most users seldom change their jobs, and each of their event sequences contains 1 - 6 events in general. 
The second is the MIMIC-III dataset. 
It contains 2,371 patients, each with more than two admissions in a hospital~\cite{johnson2016mimic}. 
The dataset covers 2,789 kinds of diseases, but each patient suffers from extremely few of them and has a limited number of admissions. 
Given these two datasets, we apply our RAML-HOT method to learn GHP models and compare the models with state-of-the-art point process models. 
Specifically, we consider six baselines: the classic Hawkes process (\textbf{HP})~\cite{zhou2013learning}, the time-varying Hawkes process (\textbf{TVHP})~\cite{xu2017learning}, the recurrent marked temporal point process (\textbf{RMTPP})~\cite{du2016recurrent}, the neural Hawkes process (\textbf{NHP})~\cite{mei2017neural}, the self-attentive Hawkes process (\textbf{SAHP})~\cite{zhang2020self}, and the transformer Hawkes process (\textbf{THP})~\cite{zuo2020transformer}. 
For our GHP model, we implement it to generate classic Hawkes processes ($i.e.$, \textbf{GHP}$_{\text{HP}}$) and extend it to generate time-varying Hawkes processes ($i.e.$, \textbf{GHP}$_{\text{TVHP}}$).

\begin{table}[t]
\centering
\begin{small}
    \caption{Comparisons on real-world data}\label{tab:real}
    \begin{tabular}{
    @{\hspace{3pt}}l@{\hspace{3pt}}|
    l@{\hspace{4pt}}l|
    l@{\hspace{4pt}}l}
    \hline\hline
    \multirow{2}{*}{Method} 
    &\multicolumn{2}{c|}{LinkedIn}
    &\multicolumn{2}{c}{MIMIC-III}\\
    \cline{2-5}
    &NLL &$d_{ot}(\widehat{\mathcal{N}},\mathcal{N})$
    &NLL &$d_{ot}(\widehat{\mathcal{N}},\mathcal{N})$\\
    \hline
    HP
    &144.45$_{\pm \text{20.70}}$ &9.29$_{\pm \text{1.38}}$
    &87.72$_{\pm \text{7.73}}$ &10.30$_{\pm \text{0.69}}$\\
    TVHP
    &113.82$_{\pm \text{8.34}}$ &8.66$_{\pm \text{1.57}}$
    &63.25$_{\pm \text{3.08}}$ &10.06$_{\pm \text{0.63}}$\\
    RMTPP
    &127.39$_{\pm \text{13.44}}$ &8.83$_{\pm \text{1.49}}$
    &82.46$_{\pm \text{6.18}}$ &11.76$_{\pm \text{0.54}}$\\
    NHP
    &52.58$_{\pm \text{14.52}}$ &7.47$_{\pm \text{1.26}}$
    &60.05$_{\pm \text{5.27}}$ &9.98$_{\pm \text{0.71}}$\\
    SAHP
    &38.91$_{\pm \text{10.33}}$ &7.09$_{\pm \text{0.80}}$
    &54.45$_{\pm \text{3.12}}$ &10.01$_{\pm \text{0.95}}$\\
    THP
    &30.64$_{\pm \text{7.03}}$ &6.44$_{\pm \text{0.61}}$
    &42.08$_{\pm \text{5.26}}$ &9.85$_{\pm \text{0.88}}$\\
    \hline
    GHP$_{\text{HP}}$
    &19.36$_{\pm \text{2.97}}$ &5.23$_{\pm \text{0.28}}$
    &33.79$_{\pm \text{6.54}}$ &9.36$_{\pm \text{2.45}}$\\
    GHP$_{\text{TVHP}}$
    &\textbf{17.55}$_{\pm \text{2.61}}$ &\textbf{4.71}$_{\pm \text{0.15}}$
    &\textbf{31.63}$_{\pm \text{5.83}}$ &\textbf{8.96}$_{\pm \text{2.29}}$\\
    \hline\hline
    \end{tabular}
\end{small}
\end{table}

For each dataset, we train the models above based on 80\% sequences and test them on the remaining 20\% sequences based on two measurements. 
Firstly, for each model we can simulate a set of event sequences and calculate their optimal transport distance to the testing set, $i.e.$, $d_{ot}(\widehat{\mathcal{N}},\mathcal{N})$.
Secondly, given the learned method, we can calculate the negative log-likelihood (NLL) of the testing sequences. 
When calculating $d_{ot}(\widehat{\mathcal{N}},\mathcal{N})$, our GHP models apply the HOT distance based on (\ref{eq:ot_tpp2},~\ref{eq:ot_cp}). 
The optimal transport $\bm{Q}^*=[q^*(\bm{N}_k, \bm{N}_l)]$ derived by (\ref{eq:ot_tpp2}) helps match the simulated sequences with the testing ones. 
For each pair of the sequence, the optimal transport $\bm{T}^*=[T^*_{uv}]$ derived by (\ref{eq:ot_cp}) indicates the correspondence between the event types of the testing sequence in the latent event type space, $i.e.$, the latent event type $\{x_1,...,x_{|\mathcal{V}|}\}\in\Omega$ for the real-world event types $\mathcal{V}$. 
For the $v$-th event type of the $l$-th testing sequence $\bm{N}_l$, we first estimate the probability that it matches with the $u$-th latent event type of the $k$-th generated sequence $\hat{\bm{N}}_k$ as $p(x_u^k|v) \propto T^*_{uv}q^*(\bm{N}_k, \bm{N}_l)$.
Then, we take $\{x_u^k\}$ as landmarks on $\Omega$ and approximate the probability density $p(x|v)$ by the kernel density estimation, $i.e.$, $p(x|v)=\frac{1}{Z}\sum_{u,k}p(x_u^k|v)\exp(-\frac{|x-x_u^k|^2}{2\sigma^2})$, where $Z$ is the normalizing constant and $\sigma$ is the bandwidth of the Gaussian kernel.
For each event type in the testing sequence, we select its latent event type corresponding to the largest $p(x|v)$, $i.e.$, $x^*=\max_x p(x|v)$.
Given the latent event types, we obtain the Hawkes process from our GHP model and calculate the NLL of the testing sequence. 
Table~\ref{tab:real} shows the performance of various models in $10$ trials. 
In particular, the baselines are learned as a single point process with a huge number of event types from sparse event sequences, which have a high risk of over-fitting.
Our GHP models, on the contrary, describe each sparse event sequence by a small point process sampled from an underlying graphon and learn the point processes jointly. 
As a result, we can find that our GHP models outperform the baselines consistently.

\begin{figure}[t]
    \centering
    \subfigure[Linkedin]{
    \includegraphics[height=4.6cm]{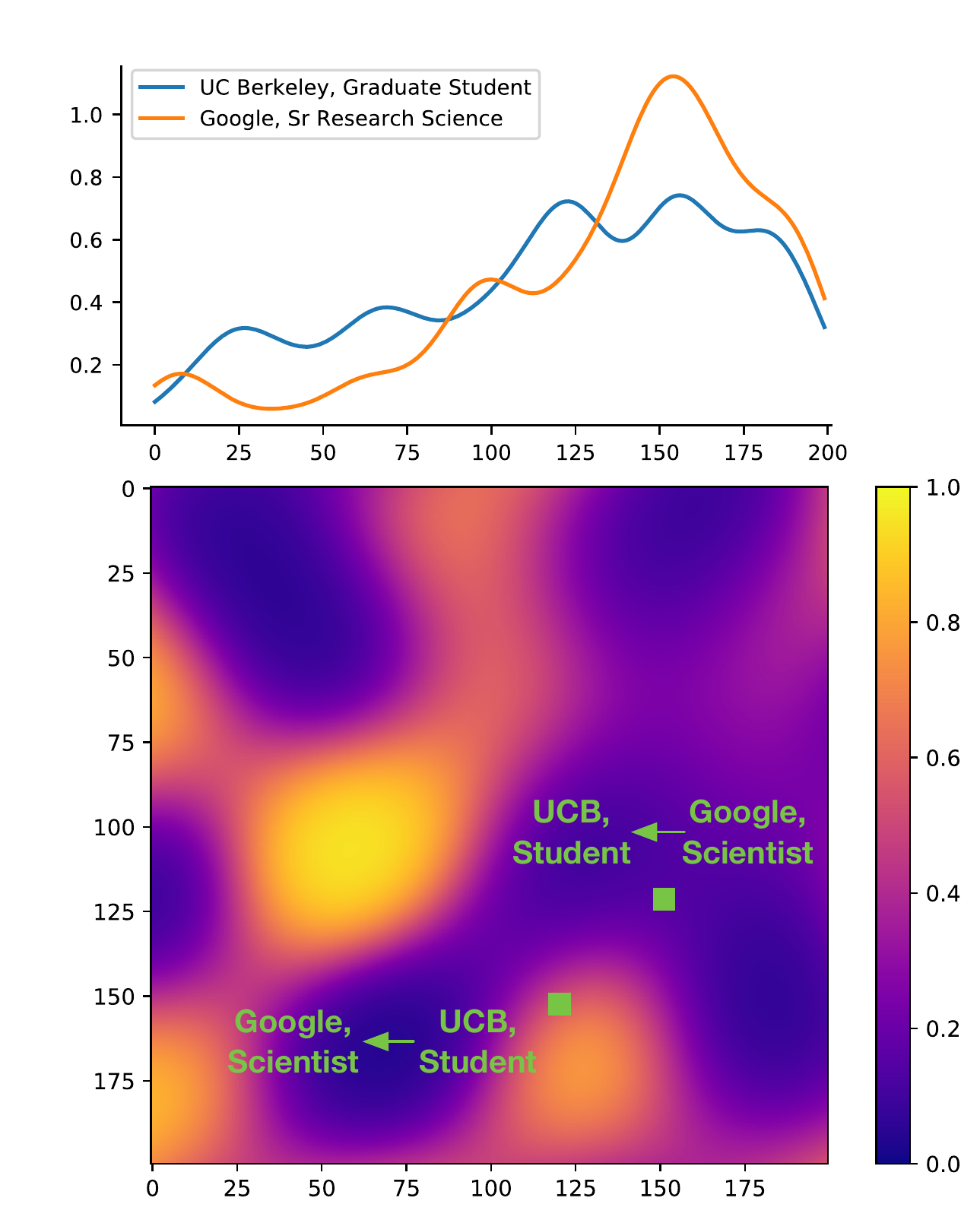}\label{fig:linkedin}
    }
    \subfigure[MIMIC-III]{
    \includegraphics[height=4.6cm]{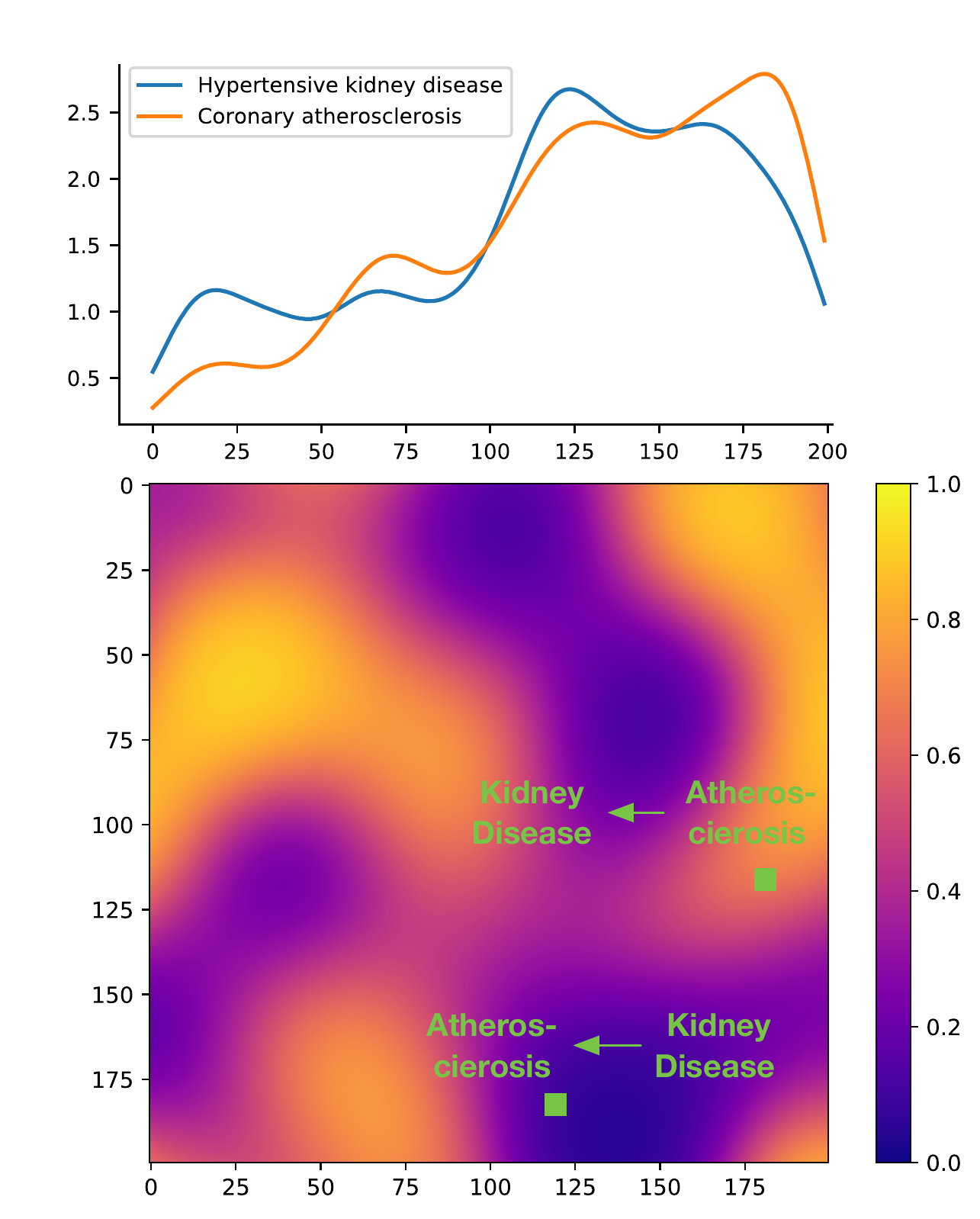}\label{fig:mimic3}
    }
    \vspace{-4mm}
    \caption{The graphons of the two real-world datasets.}
    \label{fig:real}
\end{figure}

In Figure~\ref{fig:real}, we show the probability densities of some representative real-world event types in the latent space and check their triggering patterns on the graphons. 
The graphons are visualized the resolution $200\times 200$. 
For the Linkedin dataset, we find the pairs of ``UCB, graduate student'' and ``Google, research scientist'' in the graphon according to their probability density. 
The values of the pairs indicate that a graduate student at UC Berkeley is likely to be a researcher at Google, while a researcher at Google may not be go back to school, which we think is reasonable in practice.
For the MIMIC-III dataset, we find the pairs of ``Hypertensive kidney disease'' and ``Atherosclerosis'' in the graphon. 
The values of the pairs reflect the following facts: if a patient has atherosclerosis, his risk of having the kidney disease caused by hypertensive will increase; however, the reasons for the hypertensive kidney disease are complicated, and a patient having this disease may not have atherosclerosis. 

\section{Conclusions}
In this work, we propose a graphon-based Hawkes process model, capturing the generative mechanism of multiple point processes with graph structures from heterogeneous event sequences. 
Our GHP model is a new member of hierarchical generative models for event sequence analysis. 
To our knowledge, it makes the first attempt to combine graphon models with point processes. 
In the future, we will improve GHP and its learning algorithm, $e.g.$, developing efficient algorithms to compute the HOT distance with lower complexity and building the GHP model based on deep neural networks.
Additionally, the HOT distance used in our model provides a potential solution to align event types of heterogeneous event sequences. 
Combining GHP with existing event sequence alignment methods~\cite{xu2018learning,trouleau2019learning,luo2019fused}, we plan to develop a new framework of data {\it fusion} and {\it augmentation} for large-scale point processes.

\bibliography{ihg}
\bibliographystyle{icml2021}

\newpage
\onecolumn
\appendix
\section{The Properties of Our GHP model}\label{app:A}
\subsection{The proof of Property~\ref{prop:stationary}}
\textbf{Property~\ref{prop:stationary}}
\emph{$\text{HP}_\mathcal{V}(\bm{\mu},\bm{A})\sim \text{GHP}_\Omega(f, g)$ is asymptotically stationary as long as $|\mathcal{V}|\leq V_{\max}$.} 

\begin{proof}
For a classical shift-invariant Hawkes process, its intensity function is
\begin{eqnarray}\label{eq:hp}
\begin{aligned}
\lambda_v(t)=\mu_v + \sideset{}{_{t_i<t}}\sum \phi_{vv'}(t - t_i)
=\mu_v +\sideset{}{_{t_i<t}}\sum a_{vv'}\eta(t - t_i)~\text{for}~v\in\mathcal{V}.
\end{aligned}
\end{eqnarray}
We can construct a matrix $\bm{\Phi}=[\phi_{vv'}]\in\mathbb{R}^{V\times V}$, whose element is $\phi_{vv'}=\int_{0}^{\infty}\phi_{vv'}(t)\text{d}t$. 
According to the Proposition~1 in~\cite{bacry2015hawkes}, the Hawkes process is asymptotically stationary if the impact functions satisfy: 
\begin{eqnarray}\label{eq:stable_cond}
\begin{aligned}
1)~\forall~v,v'\in\mathcal{V},~\phi_{vv'}(t)
\begin{cases}
\geq 0, &t\geq 0,\\
= 0, &t< 0.
\end{cases}
\quad 2)~\|\bm{\Phi}\|_2 < 1.
\end{aligned}
\end{eqnarray}

When setting $\phi_{vv'}(t)=a_{vv'}\kappa(t)$, as shown in (\ref{eq:hp}), we have 
\begin{eqnarray}\label{eq:phi_a}
\begin{aligned}
\bm{\Phi}=D\bm{A},~\text{where}~D=\int_{0}^{\infty}\eta(t)\text{d}t.
\end{aligned}
\end{eqnarray}
In the generative process of our GHP model, we have
\begin{eqnarray}
\begin{aligned}
a_{vv'}=\frac{1}{V_{\max} D}g(x_v, x_{v'}),~\text{where}~V_{\max}\geq V~\text{and}~g:\Omega\mapsto [0, 1).
\end{aligned}
\end{eqnarray}
Accordingly, we have
\begin{eqnarray}\label{ineq:a}
\begin{aligned}
\|\bm{A}\|_2\leq \|\bm{A}\|_F = \Bigl(\sideset{}{_{v,v'}}\sum a_{vv'}^2\Bigr)^{\frac{1}{2}} = \frac{1}{V_{\max} D}\Bigl(\sideset{}{_{v,v'}}\sum g^2(x_v,x_{v'})\Bigr)^{\frac{1}{2}}< \frac{V}{V_{\max} D}\leq \frac{1}{D}.
\end{aligned}
\end{eqnarray}
Here, the first inequality is based on the relationship between a matrix's spectral norm and its Frobenius norm. 
The second (strict) inequality is based on Assumption~\ref{assum:lip} ($i.e.$, $0\leq g(x, y)<1$ $\forall~(x,y)\in\Omega^2$). 
Plugging (\ref{ineq:a}) into (\ref{eq:phi_a}), we have $\|\bm{\Phi}\|_2<1$, thus the stability condition in (\ref{eq:stable_cond}) is satisfied.
\end{proof}

\subsection{The proof of Property~\ref{prop:lip}}~\label{app:A2}
Before proving Property~\ref{prop:lip}, we first introduce the definition of the discrete Wasserstein distance (the earth mover's distance) and that of the discrete Gromov-Wasserstein distance.
\begin{definition}[Earth Mover's Distance]\label{def:emd}
Given $\bm{a}=\{\bm{a}_m\in\mathbb{R}^D\}_{m=1}^{M}$ and $\bm{b}=\{\bm{b}_n\in\mathbb{R}^{D}\}_{n=1}^{N}$, the discrete Wasserstein distance between them is
\begin{eqnarray}\label{eq:emd}
\begin{aligned}
d_{\text{w}}(\bm{a},\bm{b}):=\sideset{}{_{\bm{T}\in\Pi\left(\frac{1}{M}\bm{1}_M,\frac{1}{N}\bm{1}_N\right)}}\min\Bigl(\sideset{}{_{m,n}}\sum t_{mn}\|\bm{a}_m-\bm{b}_n\|_2^2\Bigr)^{\frac{1}{2}}
=\sideset{}{_{\bm{T}\in\Pi\left(\frac{1}{M}\bm{1}_M,\frac{1}{N}\bm{1}_N\right)}}\min\langle \bm{D},\bm{T}\rangle^{\frac{1}{2}}.
\end{aligned}
\end{eqnarray}
Here, $\bm{T}=[t_{mn}]$ is a doubly stochastic matrix in the set 
\[\Pi\left(\frac{1}{M}\bm{1}_M,\frac{1}{N}\bm{1}_N\right)=\left\{\bm{T}=[t_{mn}] \; \bigg| \; t_{mn}\geq 0, \bm{T}\bm{1}_N=\frac{1}{M}\bm{1}_M,\bm{T}^{\top}\bm{1}_M=\frac{1}{N}\bm{1}_N\right\},
\]
where $\bm{1}_N$ is the $N$-dimensional all-one vector. 
$\bm{D}=[d_{mn}]$ is a distance matrix, whose element $d_{mn}=\|\bm{a}_m-\bm{b}_n\|_2^2$. 
The optimal $\bm{T}$ corresponding to the distance, $i.e.$, $\bm{T}^*=\arg\min_{\bm{T}\in\Pi(\frac{1}{M}\bm{1}_M,\frac{1}{N}\bm{1}_N)}\langle \bm{D},\bm{T}\rangle^{\frac{1}{2}}$, is the so-called optimal transport matrix.
\end{definition}

\begin{definition}[Discrete Gromov-Wasserstein Distance]\label{def:dgwd}
Given $\bm{a}=\{\bm{a}_m\in\mathbb{R}^D\}_{m=1}^{M}$ and $\bm{b}=\{\bm{b}_n\in\mathbb{R}^{D'}\}_{n=1}^{N}$, where $D$ can be different from $D'$, the discrete Gromov-Wasserstein distance between them is
\begin{eqnarray}\label{eq:gwd1}
\begin{aligned}
d_{\text{gw}}(\bm{a},\bm{b}):=\sideset{}{_{\bm{T}\in\Pi(\frac{1}{M}\bm{1}_M,\frac{1}{N}\bm{1}_N)}}\min\Bigl(\sideset{}{_{m,m',n,n'}}\sum t_{mn}t_{m'n'}|\|\bm{a}_{m}-\bm{a}_{m'}\|_2-\|\bm{b}_{n}-\bm{b}_{n'}\|_2 |^2\Bigr)^{\frac{1}{2}}.
\end{aligned}
\end{eqnarray}
Similar to the discrete Wasserstein distance, $\bm{T}=[t_{mn}]$ is a doubly stochastic matrix in the set $\Pi\left(\frac{1}{M}\bm{1}_M,\frac{1}{N}\bm{1}_N\right)$ and the optimal $\bm{T}$ corresponding to the distance is the optimal transport matrix. 
When two distance matrices $\bm{A}=\{a_{mm'}\}_{m,m'=1}^{M}$ and $\bm{B}=\{b_{nn'}\}_{n,n'=1}^{N}$ are provided directly, where $a_{mm'}=\|\bm{a}_m-\bm{a}_{m'}\|_2$ and $b_{nn'}=\|\bm{b}_{n}-\bm{b}_{n'}\|_2$, we can rewrite the discrete Gromov-Wasserstein distance equivalently as
\begin{eqnarray}\label{eq:gwd2}
\begin{aligned}
d_{\text{gw}}(\bm{A},\bm{B}):=\sideset{}{_{\bm{T}\in\Pi\left(\frac{1}{M}\bm{1}_M,\frac{1}{N}\bm{1}_N\right)}}\min\Bigl(\sideset{}{_{m,m',n,n'}}\sum t_{mn}t_{m'n'}|a_{mm'}-b_{nn'}|^2\Bigr)^{\frac{1}{2}}.
\end{aligned}
\end{eqnarray}
\end{definition}

\textbf{Property~\ref{prop:lip}}
\emph{For $\text{HP}_{\mathcal{V}}(\bm{\mu}_1,\bm{A}_1)$ and $\text{HP}_{\mathcal{U}}(\bm{\mu}_2,\bm{A}_2)\sim \text{GHP}_{\Omega}(f, g)$, where $\text{GHP}_{\Omega}(f, g)$ satisfies Assumption~\ref{assum:lip}, their parameters satisfy
\begin{eqnarray*}
\begin{aligned}
&C_{1}^f d_{\text{w}}(\bm{x}_1,\bm{x}_2)\leq d_{\text{w}}(\bm{\mu}_1,\bm{\mu}_2)\leq C_{2}^f d_{\text{w}}(\bm{x}_1,\bm{x}_2),\\
&d_{\text{w}}(\bm{A}_1,\bm{A}_2)\leq C^g d_{\text{w}}(\bm{x}_1^{\times},\bm{x}_2^{\times}),\\
&d_{\text{gw}}(\bm{A}_1,\bm{A}_2)\leq C^g d_{\text{gw}}(\bm{x}_1^{\times},\bm{x}_2^{\times}),
\end{aligned}
\end{eqnarray*}
where $\bm{x}_1=\{x_{v,1}\}_{v=1}^{|\mathcal{V}|}$ and $\bm{x}_2=\{x_{u,2}\}_{u=1}^{|\mathcal{U}|}$ are the latent event types, and $\bm{x}_1^{\times}=\{[x_{v,1};x_{v',1}]\}_{v,v'=1}^{|\mathcal{V}|}$ and $\bm{x}_2^{\times}=\{[x_{u,2};x_{u',2}]\}_{u,u'=1}^{|\mathcal{U}|}$ enumerate the pairs of the latent event types.}

\begin{proof}
Denote $|\mathcal{V}|=V$ and $|\mathcal{U}|=U$. 
Denote $\bm{T}^{x}$ as the optimal transport matrix corresponding to $d_{\text{w}}(\bm{x}_1,\bm{x}_2)$, $i.e.$, $\bm{T}^{x}=\arg\min_{\bm{T}\in\Pi\left(\frac{1}{V}\bm{1}_V,\frac{1}{U}\bm{1}_U\right)}(\sum_{v,u}t_{vu}|x_{v,1} - x_{u,2}|^2)^{\frac{1}{2}}=\arg\min_{\bm{T}\Pi\left(\frac{1}{V}\bm{1}_V,\frac{1}{U}\bm{1}_U\right)}\langle\bm{D}_x,\bm{T}\rangle^{\frac{1}{2}}$, where $\bm{D}_x=[d^x_{vu}]$ and $d^x_{vu}=|x_{v,1} - x_{u,2}|^2$. 
Similarly, denote $\bm{T}^{\mu}$ as the optimal transport matrix corresponding to $d_{\text{w}}(\bm{\mu}_1,\bm{\mu}_2)$, $i.e.$, $d_{\text{w}}(\bm{\mu}_1,\bm{\mu}_2)=\langle\bm{D}_{\mu},\bm{T}^{\mu}\rangle$, where $\bm{D}_{\mu}=[d^{\mu}_{vu}]$ and $d^{\mu}_{vu}=|\mu_{v,1} - \mu_{u,2}|^2$.
Obviously, we have
\begin{eqnarray}\label{eq:cond_t}
\forall~\bm{T}\in \Pi\left(\frac{1}{V}\bm{1}_V,\frac{1}{U}\bm{1}_U\right),\quad
\langle\bm{D}_x,\bm{T}\rangle^{\frac{1}{2}}\geq \langle\bm{D}_x,\bm{T}^x\rangle^{\frac{1}{2}},\quad\text{and}\quad
\langle\bm{D}_{\mu},\bm{T}\rangle^{\frac{1}{2}}\geq \langle\bm{D}_{\mu},\bm{T}^{\mu}\rangle^{\frac{1}{2}}.
\end{eqnarray}
Based on (\ref{eq:cond_t}), we have
\begin{eqnarray}\label{eq:lip_mu}
\begin{aligned}
d_{\text{w}}(\bm{\mu}_1,\bm{\mu}_2)&=\Bigl(\sideset{}{_{m,n}}\sum t_{vu}^{\mu}|\mu_{v,1} - \mu_{u,2}|^2\Bigr)^{\frac{1}{2}}\leq \Bigl(\sideset{}{_{v,u}}\sum t_{vu}^{x}|\mu_{v,1} - \mu_{u,2}|^2\Bigr)^{\frac{1}{2}}\\
&\leq C_2^f\Bigl(\sideset{}{_{v,u}}\sum t_{vu}^{x}|x_{v,1} - x_{u,2}|^2\Bigr)^{\frac{1}{2}} = C_2^f d_{\text{w}}(\bm{x}_1,\bm{x}_2).\\
d_{\text{w}}(\bm{\mu}_1,\bm{\mu}_2)&=\Bigl(\sideset{}{_{v,u}}\sum t_{vu}^{\mu}|\mu_{v,1} - \mu_{u,2}|^2\Bigr)^{\frac{1}{2}}\geq 
C_1^f\Bigl(\sideset{}{_{v,u}}\sum t_{vu}^{\mu}|x_{v,1} - x_{u,2}|^2\Bigr)^{\frac{1}{2}}\\
&\geq C_1^f\Bigl(\sideset{}{_{v,u}}\sum t_{vu}^{x}|x_{v,1} - x_{u,2}|^2\Bigr)^{\frac{1}{2}}=C_1^f d_{\text{w}}(\bm{x}_1,\bm{x}_2).
\end{aligned}
\end{eqnarray}

For $\bm{A}_1=[a_{vv'}^1]$ and $\bm{A}_2=[a_{uu'}^2]$, their discrete Wasserstein distance is
\begin{eqnarray}
\begin{aligned}
d_{\text{w}}(\bm{A}_1,\bm{A}_2)=\sideset{}{_{\bm{T}\in\Pi(\frac{1}{V^2}\bm{1}_{V^2},\frac{1}{U^2}\bm{1}_{U^2})}}\min\langle\bm{D}_A,\bm{T}\rangle^{\frac{1}{2}}=\langle \bm{D}_A, \bm{T}^A\rangle^{\frac{1}{2}}
\end{aligned}
\end{eqnarray}
where $\bm{D}_A=[d_{vv',uu'}]\in\mathbb{R}^{V^2\times U^2}$, $d_{vv',uu'}=|a_{vv'}^1 - a_{uu'}^2|^2$, and $\bm{T}^A$ is the optimal transport matrix. 
Their discrete Gromov-Wasserstein distance is
\begin{eqnarray}
\begin{aligned}
d_{\text{gw}}(\bm{A}_1,\bm{A}_2)=\sideset{}{_{\bm{T}\in\Pi(\frac{1}{V}\bm{1}_{V},\frac{1}{U}\bm{1}_{U})}}\min\langle\bm{D}_A,\bm{T}\otimes\bm{T}\rangle^{\frac{1}{2}}=\langle \bm{D}_A, \bm{T}_{\text{gw}}^A\otimes\bm{T}_{\text{gw}}^A\rangle^{\frac{1}{2}},
\end{aligned}
\end{eqnarray}
where $\bm{T}_{\text{gw}}^A$ is the optimal transport matrix and $\otimes$ represents the Kronecker multiplication between two matrices.

Similarly, we represent 
\[d_{\text{w}}(\bm{x}_1^{\times}, \bm{x}_2^{\times})=\min_{\bm{T}\in\Pi(\frac{1}{V^2}\bm{1}_{V^2},\frac{1}{U^2}\bm{1}_{U^2})}\langle\bm{D}_{X}, \bm{T}\rangle^{\frac{1}{2}}, \quad d_{\text{gw}}(\bm{x}_1^{\times}, \bm{x}_2^{\times})=\min_{\bm{T}\in\Pi(\frac{1}{V}\bm{1}_{V},\frac{1}{U}\bm{1}_{U})}\langle\bm{D}_{X}, \bm{T}\otimes\bm{T}\rangle^{\frac{1}{2}},\] respectively, where $\bm{D}_{X}=[D_{vv',uu'}]$ and $D_{vv',uu'}=\|[x_{v,1}; x_{v',1}] - [x_{u,2}; x_{u',2}]\|_2^2$. 
Accordingly, we denote $\bm{T}^X$ as the optimal transport matrix of $d_{\text{w}}(\bm{x}_1^{\times}, \bm{x}_2^{\times})$ and $\bm{T}_{\text{gw}}^X$ as the optimal transport matrix of $d_{\text{gw}}(\bm{x}_1^{\times}, \bm{x}_2^{\times})$.

Similar to the derivation shown in (\ref{eq:lip_mu}), we have
\begin{eqnarray}
\begin{aligned}
& \langle \underbrace{\bm{D}_A,\bm{T}^A\rangle^{\frac{1}{2}}}_{d_{\text{w}}(\bm{A}_1,\bm{A}_2)}\leq \langle \bm{D}_A,\bm{T}^X\rangle^{\frac{1}{2}}\leq C^g\underbrace{\langle\bm{D}_X,\bm{T}^X\rangle^{\frac{1}{2}}}_{d_{\text{w}}(\bm{x}_1^{\times}, \bm{x}_2^{\times})},\\
& \langle \underbrace{\bm{D}_A,\bm{T}_{\text{gw}}^A\otimes\bm{T}_{\text{gw}}^A\rangle^{\frac{1}{2}}}_{d_{\text{gw}}(\bm{A}_1,\bm{A}_2)}\leq \langle \bm{D}_A,\bm{T}_{\text{gw}}^X\otimes\bm{T}_{\text{gw}}^X\rangle^{\frac{1}{2}}\leq C^g\underbrace{\langle\bm{D}_X,\bm{T}_{\text{gw}}^X\otimes\bm{T}_{\text{gw}}^X\rangle^{\frac{1}{2}}}_{d_{\text{gw}}(\bm{x}_1^{\times}, \bm{x}_2^{\times})}.
\end{aligned}
\end{eqnarray}
\end{proof}

\subsection{The proof of Property~\ref{prop:error}}
\begin{definition}[Average Intensity~\cite{bacry2015hawkes}]\label{def:avg_lambda}
For the stationary Hawkes process defined in (\ref{eq:hp}), its counting process is denoted as $\bm{N}(t)=\{N_v(t)\}_{v\in\mathcal{V},t\in[0,T]}$, where $N_v(t)$ is the number of the type-$v$ events till time $t$, and its average intensity is
\begin{eqnarray}\label{eq:average}
\bar{\bm{\lambda}}:=\frac{\mathbb{E}[\text{d}\bm{N}(t)]}{\text{d}t}=(\bm{I}_V - \bm{\Phi})^{-1}\bm{\mu}=(\bm{I}_V - D\bm{A})^{-1}\bm{\mu}.
\end{eqnarray}
\end{definition}
According to Campbell's theorem~\cite{chiu2013stochastic}, given $\bar{\bm{\lambda}}=[\bar{\lambda_v}]$, we have
\begin{eqnarray}
\mathbb{E}[\text{d}N_v(t)]=T\bar{\lambda}_v=\int_{0}^{T}\lambda_v(t)\text{d}t,~\forall~v\in\mathcal{V}.
\end{eqnarray}
In other words, the average intensity reflects the overall dynamics of different event types.

The proof of Property~\ref{prop:error} is based on the theory of optimal transport and some well-known theorems.
\begin{property}[Triangle inequality~\cite{villani2008optimal}]\label{prop:tri}
For arbitrary $\bm{a}=\{\bm{a}_l\in\mathbb{R}^D\}_{l=1}^{L}$, $\bm{b}=\{\bm{b}_m\in\mathbb{R}^D\}_{m=1}^{M}$, and $\bm{c}=\{\bm{c}_n\in\mathbb{R}^D\}_{n=1}^{N}$, we have
\begin{eqnarray}
d_{\text{w}}(\bm{a},\bm{c})\leq d_{\text{w}}(\bm{a},\bm{b}) + d_{\text{w}}(\bm{b},\bm{c}).
\end{eqnarray}
\end{property}

\begin{theorem}[One-dimensional Earth Mover's Distance~\cite{rabin2011wasserstein}]\label{the:emd1d}
For two sets of 1D points, $i.e.$, $\bm{a}=\{a_n\in\mathbb{R}\}_{n=1}^{N}$ and $\bm{b}=\{b_n\in\mathbb{R}\}_{n=1}^{N}$, their earth mover's distance has a closed form solution with complexity $\mathcal{O}(N\log N)$.
\begin{eqnarray}
\begin{aligned}
d_{\text{w}}(\bm{a},\bm{b})=\frac{1}{\sqrt{N}}\|\text{\rm sort}(\bm{a}) - \text{\rm sort}(\bm{b})\|_2=\frac{1}{\sqrt{N}}\|\bm{a}-\bm{P}\bm{b}\|_2,
\end{aligned}
\end{eqnarray}
where $\text{\rm sort}(\cdot)$ sorts the elements of a vector in a descending order, and $\bm{P}\in \{\bm{P}\in \{0, 1\}^{N\times N}~|~\bm{P}\bm{1}_N=\bm{1}_N,\bm{P}^{\top}\bm{1}_N=\bm{1}_N\}$ is a permutation matrix, mapping the $n$-th largest element of $\bm{b}$ to the $n$-th largest element of $\bm{a}$ for $n=1,...,N$.
Obviously, $\frac{1}{N}\bm{P}$ is the optimal transport matrix.
\end{theorem}

Additionally, according to the definition of the earth mover's distance, we have the following theorem:
\begin{theorem}\label{the:padzero}
For a set of 1D points, $i.e.$, $\bm{a}=\{a_m\in\mathbb{R}\}_{m=1}^{M}$, if we pad $N-M$ zeros to $\bm{a}$ and obtain $\bm{a}'=[\bm{a};\bm{0}_{N-M}]$, we have
\begin{eqnarray}
d_{\text{w}}(\bm{a},\bm{a}')\leq \sqrt{\frac{N-M}{MN}}\|\bm{a}\|_2
\end{eqnarray}
\end{theorem}
\begin{proof}
For $\bm{a}$ and $\bm{a}'$, we obtain a distance matrix $\bm{D}=[\bm{D}_a, \bm{D}_0]\in\mathbb{R}^{M\times N}$. 
Here, $\bm{D}_a=[d_{mm'}]\in\mathbb{R}^{M\times M}$ and $d_{mm'}=|a_m - a_{m'}|^2$. 
Obviously, the diagonal element $d_{mm}=0$ for $m=1,..,M$.
$\bm{D}_0=[d_{mn}]\in\mathbb{R}^{M\times (N-M)}$ and $d_{mn}=|a_m|^2$ for all $n=1,...,N-M$.
Accordingly, we can design a valid transport matrix $\hat{\bm{T}}=[\frac{1}{N}\bm{I}_M, \frac{1}{MN}\bm{1}_{M\times (N-M)}]$, such that $\hat{\bm{T}}\in\Pi\left(\frac{1}{M}\bm{1}_M,\frac{1}{N}\bm{1}_N\right)$. Therefore, we have
\begin{eqnarray*}
\begin{aligned}
d_{\text{w}}(\bm{a},\bm{a}')=\min_{\bm{T}\in\Pi\left(\frac{1}{M}\bm{1}_M,\frac{1}{N}\bm{1}_N\right)}\langle\bm{D},\bm{T}\rangle^{\frac{1}{2}}\leq \langle\bm{D},\hat{\bm{T}}\rangle^{\frac{1}{2}} = \left(\underbrace{\sideset{}{_{m}}\sum d_{mm}}_{=0} + \frac{1}{MN}\sideset{}{_{m=1}^{M}}\sum\sideset{}{_{n=1}^{N-M}}\sum |a_m|^2\right)^{\frac{1}{2}}=\sqrt{\frac{N-M}{MN}}\|\bm{a}\|_2.
\end{aligned}
\end{eqnarray*}
\end{proof}


\begin{theorem}[The perturbation theory of linear system~\cite{van1983matrix}]\label{the:noise}
Suppose that we have a linear system $\bm{A}\bm{x}=\bm{b}$, where $\bm{A}\in\mathbb{R}^{N\times N}$, $\bm{x}\in\mathbb{R}^N$, and $\bm{b}\in\mathbb{R}^N$. 
Given $\bm{A}'=\bm{A}+\bm{E}$ and $\bm{b}'=\bm{b}+\bm{\epsilon}$, where $\bm{E}$ and $\bm{\epsilon}$ are noise in the system, we denote $\bm{x}'=\bm{A}'^{-1}\bm{b}'=\bm{x}+\bm{\epsilon}_x$, where the relative error of $\bm{x}$, $i.e.$, $\frac{\|\bm{\epsilon}_x\|_2}{\|\bm{x}\|_2}$, is bounded as
\begin{eqnarray}\label{ineq:bound}
\frac{\|\bm{\epsilon}_x\|_2}{\|\bm{x}\|_2}\leq \tau_{\bm{A}}\left(\frac{\|\bm{E}\|_2}{\|\bm{A}\|_2} + \frac{\|\bm{\epsilon}\|_2}{\|\bm{\mu}\|_2}\right),
\end{eqnarray}
where $\tau_{\bm{A}}$ is the condition number of $\bm{A}$.
\end{theorem}

Based on the properties and the theorems above, we can proof Property~\ref{prop:error} as follows.

\textbf{Property~\ref{prop:error}}
\emph{For $\text{HP}_{\mathcal{V}}(\bm{\mu}_1,\bm{A}_1)$ and $\text{HP}_{\mathcal{U}}(\bm{\mu}_2,\bm{A}_2)\sim \text{GHP}_{\Omega}(f, g)$, where $\text{GHP}_{\Omega}(f, g)$ satisfies Assumption~\ref{assum:lip} and $V=|\mathcal{V}|\leq |\mathcal{U}|=U$, their average intensity vectors, $i.e.$, $\bar{\bm{\lambda}}_1$ and $\bar{\bm{\lambda}}_2$, satisfy
\[
\frac{d_{\text{w}}(\bar{\bm{\lambda}}_1,\bar{\bm{\lambda}}_2)}{\|\bar{\bm{\lambda}}_1\|_2}
\leq \frac{1}{1-D\|\bm{A}_1\|_2}\left(\frac{\sqrt{2U}C^g}{C_1^f\|\bm{I}_V-D\bm{A}_1\|_2}+\frac{1}{\|\bm{\mu}_1\|_2}\right)\left(d_{\text{w}}(\bm{\mu}_1, \bm{\mu}_2)
+\sqrt{\frac{U-V}{V}}\|\bm{\mu}_1\|_2\right)+\sqrt{\frac{U-V}{VU}}.
\]
}
\begin{proof}
According to (\ref{eq:ghp}), our $\text{GHP}(f,g)$ model generates a Hawkes process $\text{HP}_{\mathcal{V}}(\bm{\mu},\bm{A})$ from the samples in $\Omega$. 
Denote $\bm{x}_1=\{x_{v,1}\}_{v=1}^{V}$ as the samples for $\text{HP}_{\mathcal{V}}(\bm{\mu}_1,\bm{A}_1)$ and $\bm{x}_2=\{x_{u,2}\}_{u=1}^{U}$ the samples for $\text{HP}_{\mathcal{U}}(\bm{\mu}_2,\bm{A}_2)$.
We have
\begin{eqnarray}
\begin{aligned}
&\mu_{v,1}=f(x_{v,1}),
&&a_{vv',1}=g(x_{v,1},x_{v',1})
&&\forall~ x_{v,1}\in\bm{x}_1,~\mu_{v,1}\in\bm{\mu}_1,~a_{vv',1}\in\bm{A}_1,\\
&\mu_{u,2}=f(x_{u,2}), 
&&a_{uu',2}=g(x_{u,2},x_{u',2})
&&\forall~ x_{u,2}\in\bm{x}_2,~\mu_{u,2}\in\bm{\mu}_2,~a_{uu',2}\in\bm{A}_2.
\end{aligned}
\end{eqnarray}

Because $V\leq U$, for $\text{HP}_{\mathcal{V}}(\bm{\mu}_1,\bm{A}_1)$ we pad $\bm{\mu}_1$ and $\bm{A}_1$ with zeros, $i.e.$, $\tilde{\bm{\mu}}_1=[\bm{\mu}_1;\bm{0}_{U-V}]\in\mathbb{R}^U$ and $\tilde{\bm{A}}_1=
\begin{bmatrix}
\bm{A}_1 &\bm{0}_{V\times (U-V)}\\
\bm{0}_{V\times (U-V)}^{\top} & \bm{0}_{(U-V)\times (U-V)}
\end{bmatrix}\in\mathbb{R}^{U\times U}$, such that $\tilde{\bm{\mu}_1}$ ($\tilde{\bm{A}_1}$) is as large as $\bm{\mu}_2$ ($\bm{A}_2$). 
Accordingly, in $\Omega$ we denote the samples corresponding to $\tilde{\bm{\mu}_1}$ and $\tilde{\bm{A}_1}$ as $\tilde{\bm{x}}_1=\{\tilde{x}_{v,1}\}_{v=1}^{U}$, which is constructed by padding $\bm{x}_1$ with $U-V$ zero points of $f(x)$, $i.e.$, 
\begin{eqnarray}
\tilde{\bm{x}}_1=\{\tilde{x}_{v,1}\}_{v=1}^{U}=\{x_{1,1},..,x_{V,1},\underbrace{x_{0}^f,...,x_{0}^f}_{U-V}\},
\end{eqnarray}
where $x_{0}^f$ is the unique zero point of $f(x)$ (Assumption~\ref{assum:lip}).

Because the Hawkes process generated by our GHP model is stationary (Property~\ref{prop:stationary}), according to (\ref{eq:average}) we have
\begin{eqnarray}\label{eq:linear}
(\bm{I}_V-D\bm{A}_1)\bar{\bm{\lambda}}_1=\bm{\mu}_1,\quad (\bm{I}_U-D\bm{A}_2)\bar{\bm{\lambda}}_2=\bm{\mu}_2,\quad\text{and}\quad
(\bm{I}_U-D\tilde{\bm{A}}_1)\tilde{\bm{\lambda}}_1=\tilde{\bm{\mu}}_1,
\end{eqnarray}
where $\tilde{\bm{\lambda}}_1=[\bar{\bm{\lambda}}_1;\bm{0}_{U-V}]$ is the average intensity $\bar{\bm{\lambda}}_1$ with padded zeros.

Following the notations used in the proof of Property~\ref{prop:lip}, we denote $\bm{T}^{\mu}$ as the optimal transport matrix for $d_{\text{w}}(\tilde{\bm{\mu}}_1, \bm{\mu}_2)$ and $\bm{T}^{\lambda}$ the  optimal transport matrix for $d_{\text{w}}(\tilde{\bm{\lambda}}_1, \bar{\bm{\lambda}}_2)$. 
According to Theorem~\ref{the:emd1d}, these two matrices are normalized permutation matrices, $i.e.$, $\bm{T}^{\mu}=\frac{1}{U}\bm{P}^{\mu}$ and $\bm{T}^{\lambda}=\frac{1}{U}\bm{P}^{\lambda}$. 
Then, we have
\begin{eqnarray}\label{ineq:1}
\begin{aligned}
d_{\text{w}}(\bar{\bm{\lambda}}_1,\bar{\bm{\lambda}}_2)
&\leq d_{\text{w}}(\bar{\bm{\lambda}}_1,\tilde{\bm{\lambda}}_1) + d_{\text{w}}(\tilde{\bm{\lambda}}_1,\bar{\bm{\lambda}}_2) 
&(\text{Property~\ref{prop:tri}})
\\
&\leq \sqrt{\frac{U-V}{VU}}\|\bar{\bm{\lambda}}_1\|_2 + \frac{1}{\sqrt{U}}\|\tilde{\bm{\lambda}}_1-\bm{P}^{\lambda}\bar{\bm{\lambda}}_2\|_2.
&(\text{Theorem~\ref{the:padzero} + Theorem~\ref{the:emd1d}})
\\
&\leq \sqrt{\frac{U-V}{VU}}\|\bar{\bm{\lambda}}_1\|_2 + \frac{1}{\sqrt{U}}\|\tilde{\bm{\lambda}}_1-\bm{P}^{\mu}\bar{\bm{\lambda}}_2\|_2.
&(\text{Based on}~(\ref{eq:cond_t}))
\end{aligned}
\end{eqnarray}

Because the permutation matrix $\bm{P}^{\mu}$ satisfies $\bm{P}^{\mu}(\bm{P}^{\mu})^{\top}=\bm{I}_U$, we have
\begin{eqnarray}\label{eq:linear2}
\bm{P}^{\mu}\bm{\mu}_2=\bm{P}^{\mu}(\bm{I}_U - D\bm{A}_2)\bar{\bm{\lambda}}_2=\bm{P}^{\mu}(\bm{I}_U - D\bm{A}_2)(\bm{P}^{\mu})^{\top}\bm{P}^{\mu}\bar{\bm{\lambda}}_2.
\end{eqnarray}
We can treat (\ref{eq:linear2}) as a perturbed version of the linear system $(\bm{I}_U-D\tilde{\bm{A}}_1)\tilde{\bm{\lambda}}_1=\tilde{\bm{\mu}}_1$ and obtain
\begin{eqnarray}\label{ineq:2}
\begin{aligned}
&\|\tilde{\bm{\lambda}}_1-\bm{P}^{\mu}\bar{\bm{\lambda}}_2\|_2
&
\\
&\leq \|\tilde{\bm{\lambda}}_1\|_2 \tau_{\bm{I}_U - D\tilde{\bm{A}}_1}\left( \frac{\|(\bm{I}_U-D\tilde{\bm{A}}_1)-\bm{P}^{\mu}(\bm{I}_U-D\bm{A}_2)(\bm{P}^{\mu})^{\top}\|_2}{\|\bm{I}_U-D\tilde{\bm{A}}_1\|_2}+\frac{\|\tilde{\bm{\mu}}_1 - \bm{P}^{\mu}\bm{\mu}_2\|_2}{\|\tilde{\bm{\mu}}_1\|_2} \right)
&(\text{Theorem}~\ref{the:noise})
\\
&= \|\tilde{\bm{\lambda}}_1\|_2 \tau_{\bm{I}_U - D\tilde{\bm{A}}_1}\left( \frac{D\|\tilde{\bm{A}}_1-\bm{P}^{\mu}\bm{A}_2(\bm{P}^{\mu})^{\top}\|_2}{\|\bm{I}_U-D\tilde{\bm{A}}_1\|_2}+\frac{\sqrt{U}d_{\text{w}}(\tilde{\bm{\mu}}_1, \bm{\mu}_2)}{\|\tilde{\bm{\mu}}_1\|_2} \right)
&(\text{Theorem}~\ref{the:emd1d})
\\
&\leq \|\bar{\bm{\lambda}}_1\|_2 \frac{1}{1-D\|\bm{A}_1\|_2}\left( \frac{D\|\tilde{\bm{A}}_1-\bm{P}^{\mu}\bm{A}_2(\bm{P}^{\mu})^{\top}\|_2}{\|\bm{I}_V-D\bm{A}_1\|_2}+\frac{\sqrt{U}d_{\text{w}}(\tilde{\bm{\mu}}_1, \bm{\mu}_2)}{\|\bm{\mu}_1\|_2} \right).
\end{aligned}
\end{eqnarray}
The second inequality in (\ref{ineq:2}) is because 1) $\|\tilde{\bm{\lambda}}_1\|_2=\|\bar{\bm{\lambda}}_1\|_2$; 2) $\|\tilde{\bm{\mu}}_1\|_2=\|\bm{\mu}_1\|_2$; 3) $\tau_{\bm{I}_U - D\tilde{\bm{A}}_1}=\frac{\sigma_{\max}(\bm{I}_U - D\tilde{\bm{A}}_1)}{\sigma_{\min}(\bm{I}_U - D\tilde{\bm{A}}_1)}=\frac{1-D\sigma_{\min}(\tilde{\bm{A}}_1)}{1-D\sigma_{\max}(\tilde{\bm{A}}_1)}\leq\frac{1}{1-D\|\tilde{\bm{A}}_1\|_2}=\frac{1}{1-D\|\bm{A}_1\|_2}$; and 4) $\|\bm{I}_U-D\tilde{\bm{A}}_1\|_2=1-D\sigma_{\min}(\tilde{\bm{A}}_1)\geq 1-D\sigma_{\min}(\bm{A}_1)=\|\bm{I}_V-D\bm{A}_1\|_2$, where $\sigma_{\min}$ ($\sigma_{\max}$) represents the minimum (the maximum) eigenvalue of a matrix. 

For the $\|\tilde{\bm{A}}_1-\bm{P}^{\mu}\bm{A}_2(\bm{P}^{\mu})^{\top}\|_2$ in (\ref{ineq:2}), we have
\begin{eqnarray}\label{ineq:3}
\begin{aligned}
&\|\tilde{\bm{A}}_1-\bm{P}^{\mu}\bm{A}_2(\bm{P}^{\mu})^{\top}\|_2
&
\\
&=\Bigl(\sideset{}{_{v,v'=1}^{U}}\sum\sideset{}{_{u,u'=1}^{U}}\sum|\tilde{a}_{vv',1} - a_{uu',2}|^2 p_{vu}^{\mu} p_{v'u'}^{\mu}\Bigr)^{\frac{1}{2}}
&
\\
&\leq C^g\Bigl(\sideset{}{_{v,v'=1}^{U}}\sum\sideset{}{_{u,u'=1}^{U}}\sum\|[\tilde{x}_{v,1};\tilde{x}_{v',1}] - [x_{u,2};x_{u',2}]\|_2^2 p_{vu}^{\mu} p_{v'u'}^{\mu}\Bigr)^{\frac{1}{2}}
&(\text{Property}~\ref{prop:lip})
\\
&=C^g\Bigl(\sideset{}{_{v,v'=1}^{U}}\sum\sideset{}{_{u,u'=1}^{U}}\sum(|\tilde{x}_{v,1} - x_{u,2}|^2+|\tilde{x}_{v',1} - x_{u',2}|^2)p_{vu}^{\mu} p_{v'u'}^{\mu}\Bigr)^{\frac{1}{2}}
&
\\
&=C^g\left(\sum_{v,u=1}^{U}p_{vu}\sum_{v',u'=1}^{U}|\tilde{x}_{v',1} - x_{u',2}|^2 p_{v'u'}^{\mu}+C^g\sum_{v',u'=1}^{N}p_{v'u'}\sum_{v,u=1}^{N}|\tilde{x}_{v,1} - x_{u,2}|^2 p_{vu}^{\mu}\right)^{\frac{1}{2}}
&
\\
&=\sqrt{2} C^g \Bigl(U\sideset{}{_{v,u=1}^{U}}\sum|\tilde{x}_{v,1} - x_{u,2}|^2 p_{vu}^{\mu}\Bigr)^{\frac{1}{2}}
&
\\
&\leq \frac{\sqrt{2} C^g}{C_1^f}\Bigl(U\sideset{}{_{v,u=1}^{U}}\sum|\tilde{\mu}_{v,1} - \mu_{u,2}|^2 p_{vu}^{\mu}\Bigr)^{\frac{1}{2}}
&(\text{Property}~\ref{prop:lip})
\\
&= \frac{\sqrt{2U} C^g}{C_1^f}\|\tilde{\bm{\mu}}_1 - \bm{P}^{\mu}\bm{\mu}_2\|_2
&(\text{Theorem}~\ref{the:emd1d})
\\
&= \frac{\sqrt{2}U C^g}{C_1^f}d_{\text{w}}(\tilde{\bm{\mu}}_1,\bm{\mu}_2).
&
\end{aligned}
\end{eqnarray}

Plugging (\ref{ineq:3}) into (\ref{ineq:2}), we have
\begin{eqnarray}\label{ineq:4}
\begin{aligned}
&\frac{\|\tilde{\bm{\lambda}}_1-\bm{P}^{\mu}\bar{\bm{\lambda}}_2\|_2}{\|\bar{\bm{\lambda}}_1\|_2}
&\\
&\leq \frac{d_{\text{w}}(\tilde{\bm{\mu}}_1, \bm{\mu}_2)}{1-D\|\bm{A}_1\|_2}\left( \frac{\sqrt{2}UC^g}{C_1^f\|\bm{I}_V-K\bm{A}_1\|_2}+\frac{\sqrt{U}}{\|\bm{\mu}_1\|_2} \right)
&
\\
&\leq \frac{1}{1-D\|\bm{A}_1\|_2}\left( \frac{\sqrt{2}UC^g}{C_1^f\|\bm{I}_V-D\bm{A}_1\|_2}+\frac{\sqrt{U}}{\|\bm{\mu}_1\|_2} \right)(d_{\text{w}}(\tilde{\bm{\mu}}_1, \bm{\mu}_1) + d_{\text{w}}(\bm{\mu}_1, \bm{\mu}_2))
&(\text{Property}~\ref{prop:tri})
\\
&\leq \frac{1}{1-D\|\bm{A}_1\|_2}\left( \frac{\sqrt{2}UC^g}{C_1^f\|\bm{I}_V-D\bm{A}_1\|_2}+\frac{\sqrt{U}}{\|\bm{\mu}_1\|_2} \right)\left(\sqrt{\frac{U-V}{V}}\|\bm{\mu}_1\|_2+d_{\text{w}}(\bm{\mu}_1, \bm{\mu}_2)\right).
&(\text{Theorem}~\ref{the:padzero})
\end{aligned}
\end{eqnarray}

Finally, plugging (\ref{ineq:4}) into (\ref{ineq:1}), we have
\begin{eqnarray}\label{ineq:5}
\begin{aligned}
\frac{d_{\text{w}}(\bar{\bm{\lambda}}_1,\bar{\bm{\lambda}}_2)}{\|\bar{\bm{\lambda}}_1\|_2}
&\leq \sqrt{\frac{U-V}{VU}}+\frac{1}{1-D\|\bm{A}_1\|_2}\left( \frac{\sqrt{2U}C^g}{C_1^f\|\bm{I}_V-K\bm{A}_1\|_2}+\frac{1}{\|\bm{\mu}_1\|_2} \right)\left(\sqrt{\frac{U-V}{V}}\|\bm{\mu}_1\|_2+d_{\text{w}}(\bm{\mu}_1, \bm{\mu}_2)\right).
\end{aligned}
\end{eqnarray}
\end{proof}

\subsection{The proof of Corollary~\ref{coro:error_s}}
Plugging $U=V$ into (\ref{ineq:5}), we obtain
\begin{eqnarray}\label{ineq:6}
\begin{aligned}
\frac{d_{\text{w}}(\bar{\bm{\lambda}}_1,\bar{\bm{\lambda}}_2)}{\|\bar{\bm{\lambda}}_1\|_2}
&\leq \frac{d_{\text{w}}(\bm{\mu}_1, \bm{\mu}_2)}{1-D\|\bm{A}_1\|_2}\left( \frac{\sqrt{2V}C^g}{C_1^f\|\bm{I}_V-D\bm{A}_1\|_2}+\frac{1}{\|\bm{\mu}_1\|_2} \right).
\end{aligned}
\end{eqnarray}

\section{Details of Our Algorithm and Experiments}\label{app:exp}
\subsection{The implementation of $d(\widehat{N}_{u}^k, N_{v}^l)$ in (\ref{eq:ot_cp})}
Denote $\{t_1^u,...,t_I^u\}\subset [0,T]$ as the sequence corresponding to $\widehat{N}_u^k(t)$ and $\{t_1^v,...,t_J^v\}\subset [0,T]$ the sequence corresponding to $N_v^l(t)$.
Without loss of generality, we assume $I\leq J$ and calculate $d(\widehat{N}_{u}^k, N_{v}^l)$ as
\begin{eqnarray}\label{eq:ot_cp2}
\begin{aligned}
d(\widehat{N}_{u}^k, N_{v}^l) =\frac{1}{T}\int_{0}^{T}|\widehat{N}_{u}^k(t) -  N_{v}^l(t)|\text{d}t=\frac{1}{T}\sideset{}{_{i=1}^{I}}\sum |t_i^u - t_i^v| + \sideset{}{_{i=I+1}^{J}}\sum |T - t_i^v|.
\end{aligned}
\end{eqnarray}
As proven in~\cite{xiao2017wasserstein}, the distance in (\ref{eq:ot_cp2}) is a valid metric for the event sequences with a single event type. 

\subsection{The significance of $\bm{V}_{\max}$ in theory}
Suppose that we have two GHP models, whose maximum numbers of event types are $V_{\max}$ and $\hat{V}_{\max}$, respectively. 
Based on (\ref{eq:ghp}), we know that the expected numbers of event types of their event sequences are $\frac{V_{\max}}{2}$ and $\frac{\hat{V}_{\max}}{2}$. 
For the two sequences having $\frac{\hat{V}_{\max}}{2}$ and $\frac{V_{\max}}{2}$ event types, respectively, Property~\ref{prop:error} indicates that the difference between their average intensity vectors is $\mathcal{O}\left(\sqrt{\frac{|V_{\max}-\hat{V}_{\max}|}{\min\{V_{\max},\hat{V}_{\max}\}}}\right)$.
Therefore, when training our GHP model, we need to carefully set $\bm{V}_{\max}$ based on the training data. 
Empirically, we calculate the averaged number of event types per sequence, denoted as $\bar{V}$ and set $\bm{V}_{\max}=2\bar{V}$.

\subsection{The Sinkhorn scaling algorithm}
When calculating the HOT distance, we need to solve a series of optimal transport problem. All the problems can be written in the following matrix format:
\begin{eqnarray}\label{eq:lp}
\min_{\bm{T}\in \Pi(\bm{p},\bm{q})}\langle\bm{D},\bm{T}\rangle.
\end{eqnarray}
We apply the Sinkhorn scaling algorithm~\cite{cuturi2013sinkhorn} to solve this problem approximately. 
In particular, we add an entropic regularizer with a weight $\beta$ into (\ref{eq:lp}) and rewrite it as
\begin{eqnarray}\label{eq:sinkhorn}
\min_{\bm{T}\in \Pi(\bm{p},\bm{q})}\langle\bm{D},\bm{T}\rangle + \beta\langle\log\bm{T},\bm{T}\rangle.
\end{eqnarray}
Then, we can solve (\ref{eq:sinkhorn}) by the following algorithm.
\begin{algorithm}[h]
    \caption{$\min_{\bm{T}\in \Pi(\bm{p},\bm{q})}\langle\bm{D},\bm{T}\rangle + \beta\langle\log\bm{T},\bm{T}\rangle$}	
    \label{alg:sinkhorn}
	\begin{algorithmic}[1]
	    \STATE Initialize $\bm{T}^{(0)}=\bm{p}\bm{q}^{\top}$, $\bm{a}=\bm{p}$, $\bm{C}=\exp(-\frac{1}{\beta}\bm{D})$.
		\STATE \textbf{for} $j=0,...,J-1$ 
		\STATE \quad Sinkhorn iteration: $\bm{b}=\frac{\bm{q}}{\bm{C}^{\top}\bm{a}}$,  $\bm{a} = \frac{\bm{p}}{\bm{C}\bm{b}}$,
		\STATE $\bm{T} = \text{diag}(\bm{a}){C}\text{diag}(\bm{b})$.
	\end{algorithmic}
\end{algorithm}

\subsection{The implementation of $d_{\text{fgw}}(\hat{\theta},\theta)$}
Given two GHP modeled on $\Omega$, denoted as $\text{GHP}_{\Omega}(f_a, g_a)$ and $\text{GHP}_{\Omega}(f_b, g_b)$. 
Denote $\mathcal{S}_{\Omega}$ as the set of measure-preserving mappings from $\Omega$ to $\Omega$. 
Based on the theory of graphon~\cite{lovasz2012large}, the distance between these two GHP models can be measured by the $\delta_2$ distance between their parameters:
\begin{eqnarray}\label{eq:d_ghp}
d(\text{GHP}_{\Omega}(f_a, g_a), \text{GHP}_{\Omega}(f_b, g_b)) = \inf_{s \in \mathcal{S}_{\Omega}} \|f_a - f_b^{s}\|_{L_2} + \|g_a - g_b^{s}\|_{L_2}, 
\end{eqnarray}
where $f_b^s(x)=f_b(s(x))$ and $g_b^s(x,y)=g_b(s(x), s(y))$. 
We say the two GHP models are equivalent if there exists at least one map $s\in\mathcal{S}_{\Omega}$ making $d(\text{GHP}_{\Omega}(f_a, g_a), \text{GHP}_{\Omega}(f_b, g_b))=0$. 

According to the theory of optimal transport~\cite{villani2008optimal}, the first term of (\ref{eq:d_ghp}) can be implemented as the Wasserstein distance between $f_a$ and $f_b$. 
For the second term, the work in~\cite{lovasz2012large} implies that we can rewrite them as the Gromov-Wasserstein distance~\cite{memoli2011gromov} between $g_a$ and $g_b$.
Combining these two distance together leads to the Fused Gromov-Wasseserstein (FGW) distance~\cite{vayer2018fused}:
\begin{eqnarray}\label{eq:measure_fgw}
\begin{aligned}
&d_{\text{fgw}}(\text{GHP}_{\Omega}(f_a, g_a), \text{GHP}_{\Omega}(f_b, g_b))\\
:=
&\sideset{}{_{\pi\in\Pi(p, q)}}\inf\mathbb{E}_{x,x'\sim \pi}\left[|f_a(x)-f_b(x')|^2\right] + \mathbb{E}_{x,x',y,y'\sim \pi\otimes\pi}\left[|g_a(x,y)-g_b(x',y')|^2\right].
\end{aligned}
\end{eqnarray}
Here, we assume $p$ and $q$ are two uniform distribution on $\Omega$. 
In our experiment, $\text{GHP}_{\Omega}(f_a, g_a)$ and $\text{GHP}_{\Omega}(f_b, g_b)$ correspond to the ground truth model and the learning result, respectively, and we use (\ref{eq:measure_fgw}) as the measurement of the estimation error. 

In practice, we set $\Omega=[0, 1]$ and uniformly $N$ samples from it, $i.e.$, $\left\{0, \frac{1}{N}, ..., \frac{N-1}{N}\right\}$. 
Accordingly, we obtain the discrete representation of each function, $i.e.$, $\bm{f}_a=[f_{i}^a]\in\bm{R}^N$, $\bm{f}_b=[f_i^b]\in\bm{R}^N$, $\bm{G}_a=[g_{ij}^a]\in\bm{R}^{N\times N}$, and $\bm{G}_b=[g_{ij}^b]\in\bm{R}^{N\times N}$. 
Then, we obtain the discrete version of (\ref{eq:measure_fgw})
\begin{eqnarray}\label{eq:discrete_fgw}
\begin{aligned}
&\sideset{}{_{\bm{T}\in\Pi(\bm{p}, \bm{q})}}\min \sum_{i,j=1}^{N} T_{ij}|f_i^a - f_j^b|^2+ \sum_{i,i',j,j'=1}^{N} T_{ij}T_{i'j'}|g_{ii'}^a - g_{jj'}^b|^2\\
=&\sideset{}{_{\bm{T}\in\Pi(\bm{p}, \bm{q})}}\min\langle \bm{D}_f,\bm{T}\rangle +\langle \bm{D}_g - 2\bm{G}_{a}\bm{T}\bm{G}_{b}^{\top},\bm{T}\rangle,
\end{aligned}
\end{eqnarray}
where $\bm{D}_f=[|f_i^a - f_j^b|^2]$,  $\bm{D}_g=\frac{1}{N}(\bm{G}_a\odot\bm{G}_a+\bm{G}_b\odot\bm{G}_b)$, and $\odot$ is the Hadamard product. 
This problem can be solved by the proximal gradient method in~\cite{xu2019gromov}.
\begin{algorithm}[h]
\small{
    \caption{$\min_{\bm{T}\in\Pi(\bm{p}, \bm{q})}\langle \bm{D}_f,\bm{T}\rangle +\langle \bm{D}_g - 2\bm{G}_{a}\bm{T}\bm{G}_{b}^{\top},\bm{T}\rangle$}	
    \label{alg:proximal}
	\begin{algorithmic}[1]
	    \STATE Initialize $\bm{T}^{(0)}=\bm{p}\bm{q}^{\top}$, $\bm{a}=\bm{p}$
		\STATE \textbf{for} $j=0,...,J-1$ 
		\STATE \quad $\bm{C}=\exp\left(-\frac{1}{\alpha}(\bm{D}_f + \bm{D}_g-2\bm{G}_a\bm{T}^{(j)}\bm{G}_b^{\top})\right)\odot \bm{T}^{(j)}$.
		\STATE \quad Sinkhorn iteration: $\bm{b}=\frac{\bm{q}}{\bm{C}^{\top}\bm{a}}$,  $\bm{a} = \frac{\bm{p}}{\bm{C}\bm{b}}$,
		\STATE \quad $\bm{T}^{(j+1)} = \text{diag}(\bm{a})\bm{C}\text{diag}(\bm{b})$.
		\STATE \textbf{Return} $\bm{T}^{(J)}$
	\end{algorithmic}
}
\end{algorithm}

\subsection{The $d_{ot}(\widehat{\mathcal{N}},\mathcal{N})$ of baselines}
The baselines also calculate $d_{ot}(\widehat{\mathcal{N}},\mathcal{N})$ by (\ref{eq:ot_tpp2}).
However, because the event types of their generated sequences perfectly correspond to those of the testing sequences, they can calculate the distance between each pair of sequences as $d(\hat{\bm{N}}_k, \bm{N}_l)=\frac{1}{|\mathcal{V}|}\sum_{v\in\mathcal{V}}d(\widehat{N}_{v}^k, N_{v}^l)$ rather than using (\ref{eq:ot_cp}).

\end{document}